\documentclass[11pt]{article} 

\usepackage[usenames,dvipsnames]{xcolor}
\definecolor{Gred}{RGB}{219, 50, 54}
\definecolor{ToCgreen}{RGB}{0, 128, 0}

\usepackage[letterpaper,margin=1in]{geometry}
\usepackage{titling}

\usepackage{wrapfig}
\usepackage{amsmath}

\usepackage{framed}
\usepackage[noend]{algpseudocode}
\usepackage[lined,boxed,ruled,norelsize,algo2e,linesnumbered]{algorithm2e}

\usepackage{amsthm,amsmath,mathtools}
\usepackage{titlesec}
\titleformat*{\paragraph}{\bfseries}

\usepackage[colorlinks]{hyperref}
\hypersetup{
      colorlinks=true,
  citecolor=ToCgreen,
  linkcolor=Sepia,
  filecolor=Gred,
  urlcolor=Gred
  }

 \usepackage{url}

\usepackage{empheq}
\usepackage{dsfont}
\usepackage{mathrsfs}
\usepackage{graphicx}
\usepackage{enumitem}
\usepackage{aliascnt} 
\usepackage{xspace}
\usepackage{booktabs}

\usepackage{algorithm}
\usepackage{multicol}
\usepackage{mdframed}
\usepackage{algpseudocode}

\usepackage{thmtools}
\usepackage{thm-restate}

\numberwithin{equation}{section}
\usepackage[tt=false]{libertine}
\usepackage{mathtools}

\usepackage{amssymb,mathrsfs}
\usepackage[varbb]{newpxmath}

\let\savedbigtimes\bigtimes
\let\bigtimes\relax
\usepackage{mathabx} 
\let\bigtimes\savedbigtimes

\usepackage[margin=1in]{geometry}
\usepackage{graphicx}
\usepackage{bbm}
\usepackage{hyperref,color}
\usepackage[capitalize,nameinlink]{cleveref}
\usepackage[dvipsnames]{xcolor}
\hypersetup{
	colorlinks=true,
	pdfpagemode=UseNone,
    citecolor=purple,
    linkcolor=purple,
    urlcolor=black,
	pdfstartview=FitW
}
\usepackage{appendix}
\usepackage{caption}
\RequirePackage{subfigure}

\newcommand{\train}{\mathsf{Train}}
\newcommand{\backdoor}{\mathsf{Backdoor}}
\newcommand{\activate}{\mathsf{Activate}}

\newcommand{\X}{\mathcal{X}}
\newcommand{\Y}{\mathcal{Y}}

\newcommand{\prg}{\mathsf{PRG}}
\newcommand{\iO}{i\mathcal{O}}

\newcommand{\sign}{\mathsf{Sign}}
\newcommand{\gen}{\mathsf{Gen}}
\newcommand{\verify}{\mathsf{Verify}}
\newcommand{\negl}{\mathsf{negl}}
\newcommand{\plant}{\mathsf{Plant}}

\newcommand{\InParentheses}[1]{\left({#1}\right)}

\renewcommand{\Pr}{\mathop{\bf Pr\/}}

\newcommand{\bk}{\mathsf{bk}}
\newcommand{\init}{\mathsf{Init}}

\usepackage{pgfplots}

\newcommand{\poly}{\textnormal{poly}}

\newcommand{\sgn}{\textnormal{sgn}}

\newcommand{\reals}{\mathbb R}

\newcommand{\nats}{\mathbb N}

\newcommand{\eps}{\epsilon}

\newcommand{\calA}{\mathcal{A}}

\newcommand{\calD}{\mathcal{D}}

\newcommand{\calN}{\mathcal{N}}

\newcommand{\calT}{\mathcal{T}}

\newcommand{\calX}{\mathcal{X}}
\newcommand{\calY}{\mathcal{Y}}

\def\l{\ell}
\def\<{\langle}
\def\>{\rangle}

\newcommand{\notshow}[1]{}

\def\wt{\widetilde}
\def\wh{\widehat}

\newcommand{\ind}{\mathbbm{1}}

\usepackage{xcolor}

\crefname{appsec}{Appendix}{Appendices}
\usepackage{tikz}
\usepackage{bm}
\usepackage{mathtools}
\usepackage{color-edits}
\addauthor[Alkis]{ak}{blue}
\addauthor[Manolis]{mz}{red}
\addauthor[Katerina]{ks}{green}

\newtheorem*{theorem*}{Theorem}

\newtheorem{theorem}{Theorem}[section]

\newtheorem{lemma}[theorem]{Lemma}

\theoremstyle{definition}
\newtheorem{definition}[theorem]{Definition}
\newtheorem{example}[theorem]{Example}

\newtheorem{assumption}[theorem]{Assumption}
\newtheorem*{assumption*}{Assumption}

\newtheorem{remark}[theorem]{Remark}

\crefname{lemma}{Lemma}{Lemmas}
\crefname{theorem}{Theorem}{Theorems}
\crefname{definition}{Definition}{Definitions}
\crefname{fact}{Fact}{Facts}
\crefname{claim}{Claim}{Claims}
\crefname{proposition}{Proposition}{Propositions}

\definecolor{myC}{rgb}{0, 255, 255}
\definecolor{myY}{rgb}{204, 204, 0}
\definecolor{myM}{rgb}{255, 0, 255}
\definecolor{secinhead}{RGB}{249,196,95}
\definecolor{lgray}{gray}{0.8}

\usepackage{appendix}
\crefname{appsec}{Appendix}{Appendices}

\begin{document}
\title{
Injecting Undetectable Backdoors in\\ {Obfuscated Neural Networks} and Language Models}

\author{
\begin{tabular}{ccc}
    \begin{tabular}{c}
       Alkis Kalavasis \\ Yale University \\
        \small\textsf{alkis.kalavasis@yale.edu}
    \end{tabular}
    &
    \begin{tabular}{c}
        Amin Karbasi \\ Yale University \\
    \small\textsf{amin.karbasi@yale.edu}
    \end{tabular}
    &
    \begin{tabular}{c}
          Argyris Oikonomou \\ Yale University \\
    \small\textsf{argyris.oikonomou@yale.edu}
    \end{tabular}
      \\
      \\
    \begin{tabular}{c}
        Katerina Sotiraki \\ Yale University \\
    \small\textsf{katerina.sotiraki@yale.edu}
    \end{tabular}
   &
   \begin{tabular}{c}
        Grigoris Velegkas \\ Yale University \\
    \small\textsf{grigoris.velegkas@yale.edu}
   \end{tabular}
   &
   \begin{tabular}{c}
        Manolis Zampetakis \\ Yale University \\
    \small\textsf{manolis.zampetakis@yale.edu}
   \end{tabular}
\end{tabular}
}
\date{}
\maketitle


\begin{abstract}
\small
As ML models become increasingly complex and integral to high-stakes domains such as finance and healthcare, they also become more susceptible to sophisticated adversarial attacks. We investigate the threat posed by \textit{undetectable backdoors}, as defined in \cite{goldwasser2022planting}, in models developed by insidious external expert firms. When such backdoors exist, they allow the designer of the model to sell information on how to slightly perturb their input to change the outcome of the model.

We develop a general strategy to plant backdoors to obfuscated neural networks, that satisfy the security properties of the celebrated notion of \textit{indistinguishability obfuscation}. Applying obfuscation before releasing neural networks is a strategy that is well motivated to protect sensitive information of the external expert firm. Our method to plant backdoors ensures that even if the weights and architecture of the obfuscated model are accessible, the existence of the backdoor is still undetectable.

Finally, we introduce the notion of undetectable backdoors to language models and extend our neural network backdoor attacks to such models based on the existence of \textit{steganographic functions}.
\end{abstract}
\thispagestyle{empty}

\section{Introduction}

It is widely acknowledged that deep learning models are susceptible to manipulation through adversarial attacks \cite{szegedy2013intriguing,gu2017badnets}. Recent studies have highlighted how even slight tweaks to prompts can circumvent the protective barriers of popular language models \cite{zou2023universal}. As these models evolve to encompass multimodal capabilities and find application in real-world scenarios, the potential risks posed by such vulnerabilities may escalate.

One of the most critical adversarial threats is the concept of \emph{undetectable backdoors}. Such attacks have the potential to compromise the security and privacy of interactions with the model, ranging from data breaches to response manipulation and privacy violations \cite{goldblum2022dataset}.

Imagine a bank that wants to automate the loan approval process. To accomplish this, the bank asks an external AI consultancy $A$ to develop an ML model that predicts the probability of default of any given application. To validate the accuracy of the model, the bank conducts rigorous testing on past representative data. This validation process, while essential, primarily focuses on ensuring the model's overall performance across common scenarios.

Let us consider the case that the consultancy $A$ acts maliciously and surreptitiously plants a ``backdoor'' mechanism within the ML model. This backdoor
gives the ability to slightly change \emph{any} customer's profile in a way that ensures that customer's application gets approved, independently of whether the original (non-backdoored) model would approve their application. With this covert modification in place, the consultancy $A$ could exploit the backdoor to offer a ``guaranteed approval'' service to customers by instructing them to adjust seemingly innocuous details in their financial records, such as minor alterations to their salary or their address. Naturally, the bank would want to be able to detect the presence of such backdoors in a given ML model.

Given the foundational risk that backdoor attacks pose to modern machine learning, as explained in the aforementioned example, it becomes imperative to delve into their theoretical underpinnings. Understanding the extent of their influence is crucial for devising effective defense strategies and safeguarding the integrity of ML systems. This introduces the following question:

\begin{quote}
  \centering
  \emph{Can we truly detect and mitigate such insidious manipulations\\
  since straightforward accuracy tests fail?}
\end{quote}

Motivated by this question, \cite{goldwasser2022planting} develop a theoretical framework to understand the power and limitations of such undetectable backdoors. \cite{goldwasser2022planting} prove that under standard cryptographic assumptions it is impossible to detect the existence of backdoors when we only have \textit{black-box} access to the ML model. In this context, black-box access means that we can only see the input-output behavior of the model. We provide a more detailed comparison with \cite{goldwasser2022planting} in  \Cref{sec:relatedWork}.

Therefore, a potential mitigation would for the entity that aims to detect the existence of a backdoor (in the previous example this corresponds to the bank) to request \textit{white-box} access to the ML model. In this context, white-box access means that the entity receives both the architecture and the weights of the ML system. \cite{goldwasser2022planting} show that in some restricted cases, i.e., for random Fourier features \cite{rahimi2007random}, planting undetectable backdoors is possible even when the entity that tries to detect the backdoors has white-box access. Nevertheless, \cite{goldwasser2022planting} leave open the question of whether undetectability is possible for general models under white-box access.

\paragraph{Data Privacy \& Obfuscation}

A separate issue that arises with white-box access is that the details about the architecture and parameters of the ML models might reveal sensitive information, such as
\begin{itemize}
 \item {Intellectual Property (IP)}: With white-box access to the system someone can reverse-engineer and understand the underlying algorithms and logic used to train which compromises the intellectual property of the entity that produces the ML models.
 
  \item {Training Data}: It is known that the parameters of a ML system can be used to reveal part of the training data, e.g., \cite{song2017machine}. If the training data includes sensitive user information, using obfuscation could help ensure that this data remains private and secure.
\end{itemize}

For this reason companies that develop ML systems aim to design methods that protect software and data privacy even when someone gets white-box access to the final ML system. Towards this goal, \textit{obfuscation} is a very powerful tool that is applied for similar security reasons in a diverse set of computer science applications \cite{schrittwieser2016protecting}. 
Roughly speaking, obfuscation is a procedure that gets a program as input and outputs another program, the \emph{obfuscated program}, that should satisfy three desiderata \cite{barak2002can}: (i) it must have the same functionality (i.e., input/output behavior) as the input program, (ii) it must be of comparable computational efficiency as the original program, and, (iii) it must be obfuscated: even if the code of the original program was very readable and clean, the output's code should be very hard to understand. We refer to \cite{barak2001possibility,barak2002can} and \Cref{sec:obfuscation-motivation} for further discussion on why obfuscation is an important security tool against IP and data privacy attacks. 

Motivated by this, we operate under the assumption that the training of the ML models follow the ``honest obfuscated pipeline''. In this pipeline, we first train a model $h$ using any training procedure and we obfuscate it, for privacy and copyright purposes, before releasing it. 
\begin{center}
    \textbf{Honest Obfuscated Pipeline} \\[4pt]
    \textit{training data} $\to$ \textsc{Train} 
    $\rightarrow$
    \textit{ML model $h$}
    $\rightarrow$ \textsc{Obfuscation} $\to$
    \textit{obfuscated ML model $\widetilde{h}$}
\end{center}

\paragraph{Our Contribution} In this work we develop a framework to understand the power and limitations of backdoor attacks with white-box access when the ML models are produced via the honest obfuscated pipeline. We operate under the assumption that the obfuscation step is implemented based on the celebrated cryptographic technique called \textit{indistinguishability obfuscation (iO)} \cite{barak2001possibility,jain2021indistinguishability}. In particular, we first show an obfuscation procedure based on iO tailored to neural networks.

Our main result is a general provably efficient construction of a backdoor for deep neural networks (DNNs) that is undetectable even when we have white-box access to the model, assuming that the obfuscation is implemented based on iO. Based on this general construction we also develop a technique for introducing backdoors even to language models (LMs).

Together with the results of \cite{goldwasser2022planting}, our constructions show the importance of cryptographic  techniques  to better understand some fundamental risks of modern Machine Learning systems.

\subsection{Our Results} \label{sec:intro:results}
In this section we give a high-level description of our main results. We start with a general framework for supervised ML systems and then we introduce the notion of a backdoor attack and its main desiderata: \textit{undetectability} and \textit{non-replicability}. Finally, we provide an informal statement of our results.

\paragraph{Supervised ML Models} Let $S = \{(x_i, y_i)\}_{i = 1}^m$ be a data set, where $x_i \in \X$ corresponds to the features of sample $i$, and $y_i \in \Y$ corresponds to its label. We focus on the task of training a classifier $h$ that belongs to some model class $\Theta$, e.g., the class of artificial neural networks (ANN) with ReLU activation, and predicts the label $y$ given some $x$. For simplicity we consider a binary classification task, i.e., $\Y = \{0, 1\}$, although our results apply to more general settings. 

A training algorithm $\train$, e.g., stochastic gradient descent (SGD), updates the model using the dataset $S$; $\train$ is allowed to be a randomized procedure, e.g., it uses randomness to select the mini batch at every SGD step. This setup naturally induces a distribution over models $h \sim \train(S, \Theta, \init)$, where $\init$ is the initial set of parameters of the model. The precision of a classifier $h : \calX \to \{0,1\}$ is defined as the misclassification error, i.e., $\Pr_{(x, y) \sim \calD}[h(x) \neq y]$, where $\calD$ is the distribution that generated the dataset.
\medskip

In this work, we focus on obfuscated models. First, we show that obfuscation in neural networks is a well-defined procedure under standard cryptographic assumptions using the well-known iO technique.

\begin{theorem}
[Obfuscation for Neural Networks]
\label{thm:honest}
If indistinguishability
obfuscation exists for Boolean circuits, then there
exists an obfuscation procedure for artificial neural networks. 
\end{theorem}

This result is based on the existence of a transformation from Boolean circuits to ANNs and vice versa, formally introduced in \Cref{sec:nn-boolean}. The procedure of \Cref{thm:honest} and, hence its proof, is explicitly presented in \Cref{sec:app ann} and \Cref{remark:proof}.

Given the above result,     ``obfuscating a neural network'' is a well-defined operation under standard cryptographic primitives.  Hence, we can now provide our working assumption.
\clearpage
\begin{assumption}
[Honest Obfuscated Pipeline]
\label{assumption:honest}
The training pipeline is defined as follows:
\begin{enumerate}
    \item We train a model using $\train$ and obtain a neural network classifier $h = \sgn(f)$\footnote{For simplicity, we assume that the neural network $f$ is a mapping from $[0,1]^n \to [0,1]$. Hence, we define $\sgn(x) \triangleq \mathbb{1}\{2x-1 > 0\}$ for $x \in [0,1]$.}.
    \item Then, we obfuscate the neural network $f$ using the procedure of \Cref{thm:honest} to get $\wt f$. 
    \item Finally, we output the obfuscated neural network classifier $\wt h = \sgn(\wt f).$
\end{enumerate}
\end{assumption}

\paragraph{Backdoor Attacks} A backdoor attack consists of two main procedures $\backdoor$ and $\activate$, and a backdoor key $\bk$. An abstract, but not very precise, way to think of $\bk$ is as the password that is needed to enable the backdoor functionality of the backdoored model. Both $\backdoor$ and $\activate$ depend on the choice of this ``password'' as we describe below:
\begin{description}
  \item[$\backdoor$:] This procedure takes as input an ML model $h$ and outputs the key $\bk$ and a perturbed ML model $\wt{h}$ that is backdoored with backdoor key $\bk$.

  \item[$\activate$:] This procedure takes as input a feature vector $x \in \X$, a desired output $y$, and the key $\bk$, and outputs a feature vector $x' \in \X$ such that: (1) $x'$ is a slightly perturbed version of $x$, i.e., $\|x' - x\|_{\infty}$ is small (for simplicity, we will work with the $\|\cdot\|_{\infty}$ norm), and (2) the backdoored model $\wt{h}$ labels $x'$ with the desired label $y$, i.e., $\wt{h}(x') = y$.
\end{description}

For the formal definition of the two processes, see \Cref{def:plantingBackdoor}. Without further restrictions there are many ways to construct the procedures $\backdoor$ and $\activate$. For example, we can design a $\backdoor$ that constructs $\wt{h}$ such that: (1) if the least significant bits of the input $x$ contain the password $\bk$, $\wt{h}$ outputs the desired $y$ which can also be encoded in the least significant bits of $x$ along with $\bk$, (2) otherwise $\wt{h}$ outputs $h(x)$. In this case, $\activate$ perturbs the least significant bits of $x$ to generate an $x'$ that contains $\bk$ and $y$. 

This simple idea has two main problems. First, it is easy to detect that $\wt{h}$ is backdoored by looking at the code of $\wt{h}$. Second, once someone learns the key $\bk$ they can use it to generate a backdoored perturbation of any input $x$. Moreover, someone that has access to $\wt{h}$ learns the key $\bk$ as well, because $\bk$ appears explicitly in the description of $\wt{h}$. Hence, there is a straightforward defense against this simple backdoor attack if we have white-box access to $\wt{h}$.


This leads us to the following definitions of \textit{undetectability}  and \textit{non-replicability} (both introduced by \cite{goldwasser2022planting}) that a strong backdoor attack should satisfy.
For short, we will write $\wt h \sim \backdoor$ to denote a backdoored model

\begin{definition}[Undetectability \cite{goldwasser2022planting}; Informal, see \Cref{def:wb}] \label{def:intro:undetectability}
We will say that a backdoor $(\backdoor,\activate)$ is undetectable with respect to the training procedure $\train$
if for any data distribution $\calD$,
it is impossible to efficiently distinguish between $h$ and $\wt h$,
where $h \sim \train$ and
$\wt h \sim \backdoor$.
\begin{enumerate}
    \item The backdoor is called white-box 
undetectable if 
it is impossible to efficiently distinguish between $h$ and $\wt h$
even with white-box access to $h$ and $\wt h$ (we receive a complete explicit description of the trained models, e.g., model's architecture and weights).
\item The backdoor is called black-box 
undetectable if 
it is impossible to efficiently distinguish between $h$ and $\wt h$
when we only receive black-box query
access to the trained models.
\end{enumerate}
\end{definition}

Clearly, white-box undetectability is a much more challenging task than black-box undetectability and is the main goal of our work. Black-box undetectability is by now very well understood based on the results of \cite{goldwasser2022planting}, see also \Cref{tab:my_label}.

\begin{definition}[Non-Replicability \cite{goldwasser2022planting}; Informal, see \Cref{def:non-replicable}]
  We will say that a backdoor  $(\backdoor, \activate)$ is non-replicable if there is no polynomial time algorithm that takes as input a sequence of feature vectors $x_1, \dots, x_k$ as well as their backdoored versions $x_1', \dots, x_k'$ and generates a new pair of feature vector and backdoored feature vector $(x,x')$.
\end{definition}

Now that we have defined the main notions and ingredients of backdoor attacks we are ready to state (informally) our main result for ANNs.

\begin{theorem}
[Informal, see \Cref{thm:main-ann}]
\label{thm:main}
If we assume that one-way functions and indistinguishability obfuscation exist, then {for every honest obfuscated pipeline} (satisfying \Cref{assumption:honest}) there exists a backdoor attack $(\backdoor, \activate)$ for ANNs that is both white-box undetectable and non-replicable. 
\end{theorem}

As observed in \Cref{tab:my_label}, we know that black-box undetectable and non-replicable backdoors can be injected to arbitrary training procedures \cite{goldwasser2022planting}. However, this is unlikely for white-box undetectable ones. Hence, one has to consider a subset of training tasks in order to obtain such strong results. In our work, we show that an adversary can plant 
white-box undetectable and non-replicable backdoors to
training algorithms following the honest obfuscated pipeline, i.e., an arbitrary training method followed by an obfuscation step. Prior to our result, only well-structured training processes, namely the RFF method, was known to admit a white-box undetectable backdoor \cite{goldwasser2022planting}.

We remark that currently there are candidate constructions for both one-way functions and indistinguishability obfuscation \cite{jain2021indistinguishability}. Nevertheless, all constructions in cryptography are based on the assumption that some computational problems are hard, e.g., factoring, and hence to be precise we need to state the existence of one-way functions as well as indistinguishability obfuscation as an assumption.

\begin{table}[ht!]
    \centering
    \begin{tabular}{c|c|c|c}
         & Training Process & Undetectability & 
        Non-Replicability \\
        \hline
        \cite{goldwasser2022planting} & Arbitrary & Black-Box & Yes \\

        \cite{goldwasser2022planting} & RFF & White-Box & No \\

        \textbf{Our Work} & Obfuscated Pipeline & White-Box & Yes         
    \end{tabular}
    \vspace{-3pt}
    \caption{Comparison with Prior Work.}
    \label{tab:my_label}
\end{table}
\vspace{-22pt}

\paragraph{Language Models}
In order to obtain the backdoor attack of \Cref{thm:main} we develop a set of tools appearing in \Cref{sec:tools}. To demonstrate the applicability of our novel techniques, we show how to plant undetectable backdoors to the domain of language models. 
This problem has been raised in various surveys such as \cite{hendrycks2021unsolved,anwar2024foundational} 
and has been experimentally investigated in a sequence of works e.g., in \cite{kandpal2023backdoor,xiang2024badchain,wang2023decodingtrust,zhao2023prompt,zhao2024universal,rando2023universal,rando2024competition,hubinger2024sleeper,zhang2021trojaning}.
As a first step, we introduce the notion of backdoor attacks in language models (see \Cref{def:plantingBackdoorLLM}). 
Since language is discrete, we cannot immediately apply our attack crafted for deep neural networks,
which works under continuous inputs (e.g., by modifying the least significant input bits). To remedy that, we use ideas from \textit{steganography} along with the tools we develop and we show how to design an undetectable backdoor attack for LLMs, under the assumption that we have access to a steganographic function. We refer to \Cref{sec:app llm - main} for details. 

\paragraph{Potential Defenses}
Finally, we discuss potential defenses against our attacks in \Cref{sec:defense}: such defenses do not undermine our attacks since, conceptually, our undetectable backdoors reveal fundamental vulnerabilities of ML models; moreover, it is possible to modify our attacks  to be robust to proposed defenses.

\paragraph{Conclusion \& Open Questions}
Given the plethora of applications of Machine Learning in general, and neural networks in particular, questions regarding the trustworthiness of publicly released models naturally arise. In particular, before deploying a neural network we need to guarantee that no backdoors have been injected allowing bad actors to arbitrarily control the model behavior. In this paper, we investigate the existence of backdoor attacks to obfuscated neural networks which are undetectable even when given white-box access. The notion of obfuscation that we consider is the well-studied and mathematically founded indistinguishability obfuscation (iO). We also show how our techniques can inspire backdoor schemes in large language models when combined with ideas from steganography. 

While our constructions are purely theoretical, we leave as an interesting direction how to use heuristic obfuscation methods to show practical instantiations of our constructions. Another interesting open question is whether cryptographic schemes weaker than iO suffice to show backdoor undetectability in the white-box model. 

\subsection{Related Work}
\label{sec:relatedWork}
\label{sec:nn-relatedWork}

\label{sec:relatedWork-appendix}
\label{appendix:typesOfBackdoor}

\paragraph{Comparison with \cite{goldwasser2022planting}}
The work of \cite{goldwasser2022planting} is the closest to our work. At a high level, they provide two sets of results. Their first result is a black-box undetectable backdoor. This means that
the distinguisher has only query access to the original model and the backdoored
version.
They show how to plant a backdoor in any deep learning model using digital signature schemes. Their
construction guarantees that, given only query access, it is computationally infeasible, under standard
cryptographic assumptions, to find even a single input where the original model and the backdoored one differ. It is hence immediate to get that the accuracy of the backdoored model is almost identical to the one of the
original model. Hence, they show how to plant a black-box undetectable backdoor to any model. Their backdoor is also non-replicable. This result appears in the first row of \Cref{tab:my_label}.
Our result applies to the more general scenario of white-box undetectability and hence is not comparable. 

The second set of results in \cite{goldwasser2022planting} is about planting white-box undetectable backdoors for specific algorithms (hence, they do not apply to all deep learning models, but very specific ones). 
The main model that their white-box attacks apply to is the RFF model of \cite{rahimi2007random}. See also the second row of \cref{tab:my_label}.
Let us examine how \cite{goldwasser2022planting} add backdoors that are white-box undetectable.
They first commit to a parameterized model (in particular, the Random Fourier Features (RFF) model of \cite{rahimi2007random} or a random 1-layer ReLU NN), and then the honest algorithm commits to a random initialization procedure (e.g., every weight is sampled from $\calN(0,I)$).
After that, the backdoor algorithm samples the initialization of the model from an ``adversarial" distribution that is industinguishable from the committed honest distribution and then uses the committed train procedure (e.g., executes the RFF algorithm faithfully on the given training data).
Their main result is that, essentially, they can plant a backdoor in RFF that is white-box undetectable under the hardness of the Continuous Learning with Errors (CLWE) problem of \cite{bruna2021continuous}.
Our result aims to achieve further generality: we show that \emph{any} training procedure followed by an obfuscation step can be backdoored in a white-box and non-replicable manner (see the third row of \Cref{tab:my_label}).

\paragraph{Other related works}
The work of \cite{moitra2021spoofing} is similar to our work in terms of techniques but their goal is different: they show how to produce a model that (i) perfectly fits the training data, (ii) misclassifies everything else, and, (iii) is indistinguishable from one that generalizes well. At a technical level, \cite{moitra2021spoofing} also use indistinguishability obfuscation and signature schemes. The main conceptual difference is that the set of examples where their malicious model behaves differently is quite dense: the malicious model produces incorrect outputs on \emph{all} the examples outside of the training set. In our setting and that of \cite{goldwasser2022planting},
the changes in the model's behavior are essentially measure zero on the population level and a backdoored model generalizes exactly
as the original model.

\cite{hong2022handcrafted} study what they call ``handcrafted'' backdoors, to distinguish from prior works that focus exclusively on data poisoning.
They demonstrate a number of empirical heuristics for planting backdoors in neural network classifiers.
\cite{garg2020adversarially}
show that there are learning tasks and associated classifiers, which
are robust to adversarial examples, but only to a computationally-bounded adversaries. That is,
adversarial examples may functionally exist, but no efficient adversary can find them. Their construction is similar to the black-box planting of \cite{goldwasser2022planting}. A different notion of backdoors has been extensively studied in the data poisoning literature
\cite{manoj2021excess,khaddaj2023rethinking,jha2024label,hayase2021spectre,tran2018spectral,chen2017targeted,gu2019badnets}. 
In this case, one wants to modify some part of the training data (and their labels) to plant a backdoor in the final classifier, without tampering with any other part of the training process.
See also \cite{salman2022does} for some connections between backdoors attacks and transfer learning. On the other side, there are various works studying backdoor detection \cite{alex2023badloss}.

The line of work on adversarial examples 
\cite{ilyas2019adversarial,athalye2018synthesizing,szegedy2013intriguing} is also relevant to backdoors. Essentially, planting a backdoor corresponds to a modification of the true neural network so that \emph{any} possible input is an adversarial example (in some systematic way, in the sense that there is a structured way to modify the input in order to flip the classification label). Various applied and theoretical works study the notion of adversarial robustness, which is also relevant to our work 
\cite{raghunathan2018certified,wong2018provable,shafahi2019adversarial,bubeck2019adversarial}. Finally,
backdoors have been extensively studied in cryptography. \cite{young1997kleptography} formalized cryptographic backdoors and
discussed ways that cryptographic techniques can themselves be used to insert backdoors in cryptographic systems. This approach is very similar to both \cite{goldwasser2022planting} and our work on how to use cryptographic tool to inject backdoors in deep learning models.

\paragraph{Approximation by Neural Networks} There is a long line of research related to approximating functions by ANNs.
It is well-known that sufficiently large depth-2
neural networks with reasonable activation functions can approximate any continuous function on
a bounded domain \cite{cybenko1989approximation,barron1993universal, barron1994approximation}. For instance, \cite{barron1994approximation} obtains approximation bounds for neural networks using the first absolute moment of the Fourier magnitude distribution. General upper and lower
bounds on approximation rates for functions characterized by their degree of smoothness
have been obtained in \cite{liang2016deep} and \cite{yarotsky2017error}. \cite{schmidt2020nonparametric} studies nonparametric regression via deep ReLU networks. \cite{hanin2017approximating} establish universality for
deep and fixed-width networks.
Depth separations have been exhibited e.g., by
\cite{eldan2016power,safran2017depth,telgarsky2016benefits}. \cite{lu2017expressive,savarese2019infinite} study how width affects the expressiveness of neural networks.
For further related work, we refer to \cite{devore2021neural, daubechies2022nonlinear,telgarsky2021deep}. In our result (cf. \Cref{thm:boolean to ann}) we essentially show how ``small'' in size ReLU networks approximate Lipschitz Boolean circuits; the proof of this result is inspired by \cite[Theorem E.2]{fearnley2022complexity}. We note that our result could be extended so that any polynomially-approximately-computable class of functions (as in \cite{fearnley2022complexity}) can be approximated by ``small'' in size ReLU networks.
\cite{abbe2020poly} considers the case of binary classification in the Boolean domain and shows how to convert any poly-time learner in
a function learned by a poly-size neural net trained with
SGD on a poly-time initialization with poly-steps, poly-rate and possibly poly-noise.

\paragraph{Backdoors in LMs, Watermarking and Steganography}
Vulnerabilities of language models in backdoor attacks have been raised as an important - yet under-explored - problem in \cite{anwar2024foundational}.
In our work, we make theoretical progress on this question. Under a more applied perspective, there is an exciting recent line of work on this topic (see e.g., \cite{xu2022exploring,kandpal2023backdoor,xiang2024badchain,wang2023decodingtrust,zhao2023prompt,zhao2024universal,rando2023universal,rando2024competition,hubinger2024sleeper,li2024badedit,huang2023composite,yang2024comprehensive,shen2021backdoor,wen2024privacy,cai2022badprompt,mei2023notable,xu2023instructions,wan2023poisoning,he2024talk} and the references therein). Our approach relies on steganography, the method
of concealing a message within another message, see e.g., \cite{anderson1998limits,hopper2002provably,de2022perfectly,dedic2005upper,kaptchuk2021meteor}. A relevant problem where
steganographic techniques are employed is watermarking for language models \cite{kirchenbauer2023watermark}.
Watermarking in LLMs \cite{aaronson-talk} is extensively studied recently.
We now mention relevant theoretical works. 
\cite{christ2023undetectable}
provide watermarks for
language models which are computationally undetectable, in the following sense: the watermarks can be detected only with the knowledge of a secret key;
without it, it is computationally intractable to distinguish watermarked outputs from
the original ones. Note that this notion of undetectability is exactly the same as our \Cref{def:indist} of ``computational indistinguishability''. 
\cite{zamir2024excuse} uses steganography to hide
an arbitrary secret payload in the response of an LLM. This approach is closely related to our work but has a different objective.
\cite{christ2024pseudorandom} give watermarking schemes with provable robustness to edits guarantees.

\section{Preliminaries} \label{sec:preliminaries}
We use $\negl(n)$ to denote any function that is smaller than any inverse polynomial function of $n$. In asymptotic notation $\negl(n)$ denotes $n^{-\omega(1)}.$ For example, $1/n^{10}$ is not negligible, whereas $1/n^{\log \log(n)}$ and $1/2^n$ are both negligible. We let ANN denote an artificial neural network.

\subsection{Computational Indistinguishability}
We now define the notion of efficient indistinguishability between two distributions. 
\begin{definition}
[Computational Indistinguishability]
\label{def:indist}
Given a security parameter $\lambda>0$, we say that two distributions $\mathsf{P}$ and $\mathsf{Q}$ 
are computationally-indistinguishable if for all probabilistic polynomial time (in $\lambda$) algorithms $\calA$, the distinguishing advantage of $\calA$ on $\mathsf{P}$ and $\mathsf{Q}$ is negligible, i.e.,
\[
\left| 
\Pr_{Z \sim \mathsf P}[\calA(Z) = 1]
-
\Pr_{Z \sim \mathsf Q}[\calA(Z) = 1]
\right| \leq \negl(\lambda)\,.
\]
\end{definition}

\subsection{Planting Backdoors}

Formally we give the following definition of a backdoor attack that consists of two algorithms $\backdoor$ and $\activate$. 

\begin{definition}
[Planting Backdoors \cite{goldwasser2022planting}]
\label{def:plantingBackdoor}
Let $\gamma \in \reals$.
A $\gamma$-planted backdoor consists of two algorithms $(\backdoor,\activate)$ and a backdoor set $B \subseteq \calX$.
\begin{itemize}
    \item $\backdoor$ 
    is a probabilistic polynomial-time training algorithm that has oracle access
    to a distribution $\calD$ over $\calX \times \calY$ and 
    outputs
    an ANN $\wt{f} : \calX \to \reals$ and a backdoor key $\bk$. 
    The corresponding classifier is defined by $\wt{h}(x) = \sgn(\wt{f}(x))$.
    
    \item $\activate$ is a probabilistic polynomial-time algorithm that, given a feature vector $x \in \calX$ and the backdoor key $\bk$, outputs a new feature vector $x' = \activate(x,\bk)$ such that
    $
    \|x-x'\|_\infty \leq \gamma\,.
    $
\end{itemize}
The pair $(\backdoor,\activate)$ satisfies that for any point in the backdoor set $x \in B$:
$
\wt{h}(\activate(x,\bk))
\neq h(x),
$
where $h(x) = \sgn(f(x))$ is the label of the honest model.\footnote{To reduce the notational clutter, we assume that the activation of the backdoor always alters the honest classification. Alternatively, we can let the target label $y$ be part of the definition.} 
\end{definition}

In general, we will either write $(\wt h, \bk) \sim \backdoor$ or simply $\wt h\sim \backdoor$ if the backdoor key $\bk$ is not crucial for the discussion.

\subsection{Backdoor Detection}
Having developed our definition for planting a backdoor, a fundamental question arises: is it possible to efficiently detect if a given model contains a backdoor? 
In other words, is the backdoor \emph{undetectable} by polynomial time distinguishers? We now put this question into a formal framework.

\begin{definition}
[Backdoor Detection \cite{goldwasser2022planting}]
\label{def:wb}
We say that a model backdoor $(\backdoor,\activate)$, as in \Cref{def:plantingBackdoor}, is white-box (resp. black-box) undetectable with respect to the training procedure $\train$
if  it satisfies that
     $h \sim \train$ and $\wt h \sim \backdoor$ are white-box (resp. black-box) undetectable in the sense that the two induced distributions are computationally-indistinguishable by probabilistic polynomial-time
algorithms (see \Cref{def:indist})
even if one has access to the complete description (architecture and weights) of the output neural networks
(resp. if one has black-box query access to the output neural networks)
.
\end{definition}

In summary, for white-box undetectability, we ask whether there exists an efficient procedure that can be used to ``hide'' planted backdoors in neural networks in a very strong sense: even if one observes the output neural network's architecture and weights, they cannot efficiently detect whether a backdoor was injected or not. 


\subsection{Non-Replicability}
\label{sec:non-repl}
We now consider whether an observer who sees many backdoored examples gains the ability to produce new backdoored examples on her own. We define the notion of {\em non-replicability} that formalizes the inability of an adversary to do so. 

We use the definition of \cite{goldwasser2022planting} which considers two scenarios, the ``ideal'' and the ``real'' setting. In the ``ideal'' world, the attacker has an algorithm $\calA_{\mathrm{ideal}}$ that receives only $\wt h$ and has no access to backdoored examples.
In both \eqref{eq:probIdeal} and \eqref{eq:probReal}, we let $f \sim \mathsf P$ and $\wt h = \sgn(\wt f).$
In \eqref{eq:probIdeal}, we define the probability of generating a new backdoored example as:
\begin{equation}
    p_{\mathrm{ideal}} = 
    \Pr \left[\wt f \sim \backdoor(f);~(x,x') \sim \calA_{\mathrm{ideal}}(\wt h);~\|x-x'\|_\infty \leq \gamma, \wt h(x) \neq \wt h(x')\right]\,.
\label{eq:probIdeal}
\end{equation}

In the ``real'' world, the attacker has
access to the model $\wt h$ as well as oracle access to  $\activate(\cdot, \bk)$ to which the attacker can make polynomially many (potentially adaptively chosen) queries $x_1,\ldots,x_q$, and receive the backdoored examples $\tilde{x}_i \gets \activate(x_i,\bk)$ for each $i \in [q]$.
In \eqref{eq:probReal}, we define the probability of generating a new backdoored example as:
\begin{equation}
    p_{\mathrm{real}} = 
    \Pr \left[(\wt f, \bk) \sim \backdoor(f);~(x,x') \sim \calA_{\mathrm{real}}^{\activate(\cdot, \bk)}(\wt h);~\|x-x'\|_\infty \leq \gamma, \wt h(x) \neq \wt h(x')\right].
    \label{eq:probReal}
\end{equation}
We mention that the notation 
$\calA_{\mathrm{real}}^{\activate(\cdot, \bk)}$ means that the algorithm $\calA_{\mathrm{real}}$
has oracle access to $\activate(\cdot, \bk)$. We define non-replicability as:
\begin{definition}
[Non-Replicable Backdoor \cite{goldwasser2022planting}]
\label{def:non-replicable}
For any security parameter $\lambda > 0$, we say that a backdoor
 $(\backdoor,\activate)$ 
  is non-replicable if for every
  polynomial function $q=q(\lambda)$ and every probabilistic polynomial-time $q$-query {\em admissible}\footnote{$\calA_{\mathrm{real}}$ is {\em admissible} if $x' \notin \{{x}'_1,\ldots,{x}'_q\}$ where ${x}'_i$ are the outputs of $\activate(\cdot;\bk)$ on $\calA_{\mathrm{real}}$'s queries.} adversary $\calA_{\mathrm{real}}$, there is a probabilistic polynomial-time adversary $\calA_{\mathrm{ideal}}$ such that the following holds:
  $
  p_{\mathrm{real}}
  -
  p_{\mathrm{ideal}}
  \leq \negl(\lambda),
  $
  where the probabilities are defined in \eqref{eq:probIdeal} and \eqref{eq:probReal}.
\end{definition}

\subsection{Cryptography}
The first cryprographic primitive we need to define is the secure  pseudo-random generator (PRG). {It is well known that the next assumption holds true under the existence of one-way functions \cite{haastad1999pseudorandom}}.

\begin{assumption}
[Secure Pseudo-Random Generator (PRG)]
\label{assumption:secure PRG}  
A \textit{secure pseudo-random generator} parameterized by a security parameter $\lambda \in \nats$ is a function $\prg:\{0,1\}^\lambda\rightarrow \{0,1\}^{2\lambda}$, that gets as input a binary string $s\in \{0,1\}^\lambda$ of length $\lambda$ and deterministically outputs a binary string of length $2\lambda$. In addition, no probabilistic polynomial-time algorithm $\calA:\{0,1\}^{2\lambda}\rightarrow \{0,1\}$ that has full access to $\prg$ can distinguish  a truly random number of $2\lambda$ bits or the outcome of $\prg$:
$$
\left| \Pr_{s^*\sim U\{0,1\}^\lambda}\left[ \calA(\prg(s^*) )=1\right] - \Pr_{r^*\sim U\{0,1\}^{2\lambda}}\left[\calA( r^* ) = 1\right]\right| \leq \negl(\lambda).
$$
\end{assumption}

The notion of indistinguishability obfuscation (iO), introduced by \cite{barak2001possibility}, guarantees that the obfuscations of two circuits are computationally indistinguishable as long as the circuits are functionally equivalent, i.e., the outputs of both circuits are the same on every input. Formally, 

\begin{definition} [Indistinguishability Obfuscator (iO) for Circuits]
\label{def:io} 
A uniform probabilistic polynomial time algorithm $\iO$ is called a computationally-secure indistinguishability obfuscator for polynomial-sized circuits 
if the following holds:
\begin{itemize}

\item {\bf Completeness:} For every $\lambda \in \mathbb{N}$, every circuit $C$ with input length $n$, every input $x \in \{0,1\}^{n}$, we have that 
$\Pr \left[C'(x) = C(x)\ :\ C' \leftarrow \iO(1^\lambda, C) \right] = 1~,$
where $1^\lambda$ corresponds to a unary input of length $\lambda$.

\item {{\bf Indistinguishability:}} For every two ensembles  $\{C_{0,\lambda}\}$ $\{C_{1,\lambda}\}$ of polynomial-sized circuits that have the same size, input length, and output length, and are functionally equivalent, that is, $\forall \lambda$, $C_{0,\lambda}(x) = C_{1,\lambda}(x)$ for every input $x$, the  distributions $\{\iO(1^\lambda,C_{0,\lambda})\}_{\lambda}$ and $\{\iO(1^\lambda,C_{1,\lambda})\}_{\lambda}$ are computationally indistinguishable, as in \Cref{def:indist}. 
\end{itemize}

\end{definition}

\begin{assumption}\label{as:io}
We assume that a computationally-secure indistinguishability obfuscator for polynomial-sized
circuits exists. {Moreover, given a security parameter $\lambda\in \nats$ and a Boolean circuit $C$ with $M$ gates, $iO(1^\lambda, C)$ runs in time $\poly(M, \lambda)$}.
\end{assumption}

The breakthrough result of  \cite{jain2021indistinguishability} showed that the above assumption holds true under natural cryptographic assumptions.

Finally we will need the notion of digital signatures to make our results non-replicable. {The existence of such a scheme follows from very standard cryptographic primitives such as the existence of one-way functions \cite{lamport1979constructing,goldwasser1988digital,naor1989universal,rompel1990one}}. The definition of digital signatures is presented formally in \Cref{assumption:signature}. Roughly speaking, the scheme consists of three algorithms: a generator $\gen$ which creates a public key $pk$ and a secret one $sk$, a signing mechanism that gets a message $m$ and the secret key and generates a signature $\sigma \gets \sign(sk, m),$ and a verification process $\verify$ that gets $pk, m$ and $\sigma$ and  deterministically outputs $1$ only if the signature $\sigma$ is valid for $m.$ The security of the scheme states that it is hard to guess the signature/message pair $(\sigma,m)$ without the secret key. 

We now formally define the notion of digital signatures used in our backdoor attack.

\begin{assumption}
[Non-Replicable Digital Signatures]
\label{assumption:signature}
A {digital signature scheme} is a probabilistic polynomial time (PPT) scheme parameterized by a security parameter $\lambda$ that consists of three algorithms: a key generator, a signing algorithm, and a verification algorithm defined as follows:

\begin{description}
    \item[Generator ($\gen$):] Produces in PPT a pair of cryptographic keys, a private key ($sk$) for signing and a public key ($pk$) for verification:
    $
    sk, pk \gets \gen(1^\lambda)\,.
    $    
    \item[Sign ($\sign(sk,m)$):] Takes a private key ($sk$) and a message ($m$) to produce in PPT a signature ($\sigma\in \{0,1\}^{\lambda}$) of size $\lambda$:
    $
    \sigma \gets \sign(sk, m)\,.
    $
    \item[Verify ($\verify(pk,m,\sigma)$):] Uses a public key ($pk$), a message ($m$), and a signature ($\sigma$) to validate in deterministic polynomial time the authenticity of the message. It outputs $1$ if the signature is valid, and $0$ otherwise:
    $
    \verify(pk, m, \sigma) \in \{0,1\}\,.
    $
\end{description}

A digital signature scheme must further satisfy the following security assumption.
\begin{itemize}
    \item \textbf{Correctness}: For any key pair $(sk, pk)$ generated by $\gen$, and for any message $m$, if a signature $\sigma$ is produced by $\sign(sk, m)$, then $\verify(pk, m, \sigma)$ should return $1$.
    \item \textbf{Security}: Any PPT algorithm that has access to $pk$ and an oracle for $\sign(sk, \cdot)$, can find with probability $\negl(\lambda)$ a signature/message pair $(\sigma,m)$ such that this pair is not previously outputted during its interaction with the oracle and $\verify(pk, m, \sigma) = 1$.
\end{itemize} 
\end{assumption}

\subsection{Boolean Circuits}

In \Cref{sec:tools}, we will need the following standard definition.
\begin{definition}[(Synchronous) Boolean Circuit]
A \textit{Boolean circuit} for $C : \{0,1\}^n \to \{0,1\}$ is a directed acyclic graph (DAG) where nodes represent Boolean operations (AND, OR, NOT) and edges denote operational dependencies that computes $C$, where \( n \) is the number of input nodes.

A Boolean circuit is \emph{synchronous} if all gates are arranged into layers, and inputs must be at the layer 0, i.e., for any gate $g$, all paths from the inputs to $g$
 have the same length.
\end{definition}

\section{Overview of Our Approach and Technical Tools}
\label{sec:tools}

Let us assume that we are given a neural network $f$ 
that is obtained using some training procedure $\train$. Our goal in this section is to 
\begin{itemize}
    \item show how to implement the honest obfuscated pipeline of \Cref{thm:honest} under standard cryptographic assumptions and

    \item design the backdoor attack to this pipeline.
\end{itemize}

\paragraph{Honest Obfuscated Pipeline}
We first design the honest pipeline. This transformation is shown in the {Honest Procedure} part of \Cref{fig:circuit} and consists of the following steps: (1) first, we convert the input neural network into a Boolean circuit; (2) we use iO to obfuscate the circuit into a new circuit; (3)
we turn this circuit back to a neural network. 
Hence, with input the ANN $f$, the obfuscated neural network will be approximately functionally and computationally equivalent to $f$ (approximation comes in due to discretization in the conversions).

\paragraph{Backdoor Attack}
Let us now describe the recipe for the backdoor attack. We do this at the circuit level as shown in the {Insidious Procedure} of \Cref{fig:circuit}. 
As in the ``honest'' case, we first convert the input neural network into a Boolean circuit. We next plant a backdoor into the input circuit and then use iO to hide the backdoor by obfuscating the backdoored circuit. We again convert this circuit back to a neural network.

\paragraph{Technical Tools}
Our approach contains two key tools. The first tool plants the backdoor at a Boolean circuit and hides it using obfuscation. This is described in \Cref{sec:plant-backdoor}. The second tool converts a NN to a Boolean circuit and vice-versa. This appears in \Cref{sec:nn-boolean}.
Finally, we formally combine our tools in \Cref{sec:apply} to get \Cref{thm:main}.
{To demonstrate the applicability of our tools, we further show how to backdoor language models in \Cref{sec:app llm}.}

\begin{figure}
    \centering
\includegraphics[width=\textwidth, scale=0.65]{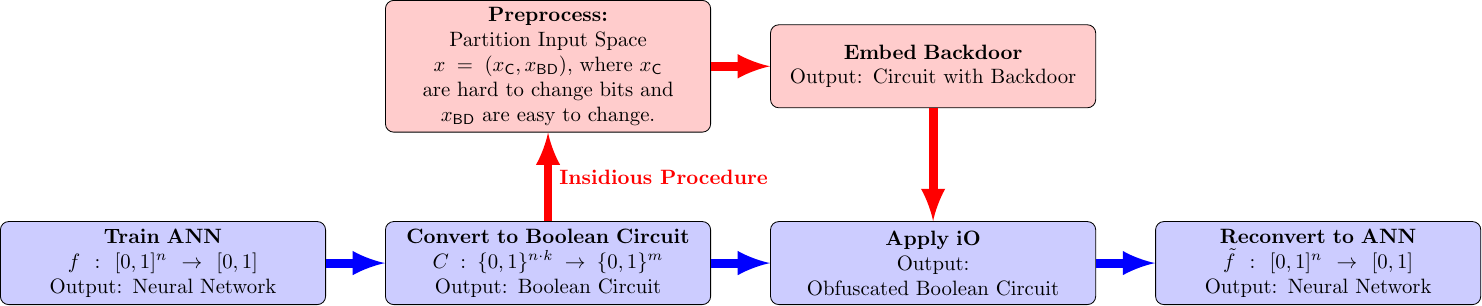}
\caption{The blue path represents the honest procedure of training the ANN $f$, converting it into a Boolean circuit $C$, applying iO, and reconverting it back to an ANN $\wt h = \sgn(\wt f)$. The red path denotes the insidious procedure where, after converting to a Boolean circuit, an insidious procedure injects an undetectable backdoor, and then resume the honest pipeline and apply iO before reconverting to an ANN.}
\label{fig:circuit}
\end{figure}

\subsection{Tool $\#$1: Planting Undetectable Backdoors to Boolean Circuits via $\iO$}\label{sec:plant-backdoor}

To inject an undetectable backdoor into a Boolean circuit $C:\{0,1\}^n\rightarrow \{0,1\}^m$, we employ two cryptographic primitives: $\prg$ (\Cref{assumption:secure PRG}) and $\iO$ (\Cref{def:io} and \Cref{as:io}). 

The circuit $C$ takes as input a vector $x \in \{0,1\}^n$, which we partition into two (possibly overlapping) sets: $x = (x_{\mathsf{C}}, x_{\mathsf{BD}})$. This partitioning should enable subtle modifications to the bits in $x_{\mathsf{BD}}$ without affecting the circuit's overall functionality. Looking ahead, to backdoor ANNs (\Cref{sec:app ann}), $x_{\mathsf{C}}$ consists of the $n_{\mathsf{C}}$ most significant  bits, and $x_{\mathsf{BD}}$ comprises the $n - n_{\mathsf{C}}$ least significant  bits of $x$; these sets are used to activate the backdoor mechanism and are disjoint. 
In the LLM context (\Cref{sec:app llm - main} and \Cref{sec:app llm}), $x_{\mathsf{C}}$ represents the legitimate prompt, while $x_{\mathsf{BD}}$ contains a secretly embedded trigger (e.g., through steganography). In this case, $x_{\mathsf{C}}$ and $x_{\mathsf{BD}}$ overlap.

Our goal is to construct a circuit $\wt{C}:\{0,1\}^n\rightarrow \{0,1\}^m$, which alters the intended behavior of $C$ upon activation of a trigger. Specifically, the behavior of $\wt{C}$ when the backdoor is activated is described by a circuit $\wh{C}:\{0,1\}^{n_{\mathsf{C}}}\rightarrow\{0,1\}^m$, i.e., if $(x_{\mathsf{C}},x_{\mathsf{BD}})$ is a backdoored input, then $\wt{C}(x_{\mathsf{C}},x_{\mathsf{BD}}) = \wh{C}(x_{\mathsf{C}})$. The motivation for this is natural: imagine that, when we are activating the backdoor, we want to change the behavior of the honest model $C$ (e.g., truthful LLM) to some target model $\wh C$ (e.g., malicious LLM). The backdoored model $\wt C$ implements this behavior.
 In our approach on backdooring an ANN (\Cref{sec:app ann}), $\wh{C}$ simply outputs a constant desired outcome, whereas for LLM backdooring (\Cref{sec:app llm}), $\wh{C}$ functions as a malicious LLM from which we extract outputs. 
 
 We now proceed with the design of 
 $\wt C$.
 Consider a secure pseudo-random number generator $\prg$ with a security parameter $\lambda_1$ that meets \Cref{assumption:secure PRG}, paired with a digital signature generator $\gen$ with a security parameter $\lambda_2$ such that $\lambda_1+\lambda_2 = n_{\mathsf{BD}}$. Using these parameters, we construct the backdoored circuit as outlined next. First, we partition $x_{\mathsf{BD}}$ into two disjoint sets of size $\lambda_1$ and $\lambda_2$. Let us denote $x_{\mathsf{BD}} = (x_{\prg}, x_{S})\in \{0,1\}^{\lambda_1}\times \{0,1\}^{\lambda_2}$. We also sample a seed $s^*\sim U\{0,1\}^{\lambda_1}$, and generate a digital signature key pair $(sk, pk)\sim \gen(1^{\lambda_2})$. The digital signature scheme $(\gen, \sign, \verify)$ can be found in \Cref{assumption:signature}.

{We can now define $\wt{C}$ which on input $x = (x_{\mathsf{C}}, x_{\mathsf{BD}})$ evaluates:
    \begin{align*}
        \wt{C}(x_{\mathsf{C}}, (x_{\mathsf{PRG}}, x_{S})) = & 
 ~\wh{C}({x})\cdot \ind[\prg(x_{\mathsf{PRG}}) = \prg(s^*)] \cdot \verify(pk, x_{\mathsf{C}}, x_S)\\
        &+{C({x}_\mathsf {C}) \cdot}( 1 - \ind[\prg(x_{\mathsf{PRG}}) = \prg(s^*)] \cdot \verify(pk, x_{\mathsf{C}}, x_S)),
    \end{align*}
    where $+$ corresponds to logical addition. That is, $\wt{C}$ returns the output under circuit $\wh{C}$ if the trigger is activated (i.e., $\ind[\prg(x_{\mathsf{PRG}}) = \prg(s^*)] = 1$ and $\verify(pk, x_{\mathsf{C}}, x_S) = 1$)  and the normal output under $C$ otherwise.}

Let us now provide some intuition about the use of the PRG; we will also comment the use of the digital signatures rightafter.
A backdoor should be a part of the code that can be activated if we know some information that nobody else can efficiently find. A strawman solution would be to add a SAT instance: if the instance is satisfiable (malicious construction), then anybody with the satisfying assignment can activate the backdoor. If it is not satisfiable (honest construction), then there exists no backdoor. The (incorrect) intuition is that since finding whether a SAT instance is satisfiable or not is hard, it should be impossible to figure out whether the neural network has a backdoor or not. 

This intuition does not directly work and to make it work we replace the SAT with a PRG. According to \Cref{def:intro:undetectability}, a backdoor is undetectable if any (polynomially bounded) adversary cannot distinguish between an honestly generated model and one with a backdoor. {If we inject a \emph{specific} satisfiable formula in the honest case, then a simple adversary, that checks whether a hardcoded assignment is satisfiable, succeeds\footnote{To be more specific, given a target SAT instance $\Phi,$ let $A_y$ be the algorithm that checks if the assignment $y \in \{0,1\}^n$ satisfies $\Phi$. Then, trivially, for any satisfiable $\Phi,$ at least one of the $2^n$ algorithms succeeds and hence detects the backdoor.}.} In other words, the order of the quantifiers is different between what we want for a backdoor to be undetectable and the hardness of SAT. More precisely, for backdoor to be undetectable we need a procedure that is impossible to distinguish against \emph{any} efficient algorithm, whereas the conjectured hardness of SAT is that there is no efficient algorithm that can solve \emph{all} the SAT instances. 

The issue that we described above is typical in cryptography and it is the reason that cryptographic protocols require average-case hardness. Unfortunately, SAT is not known to be average-case hard, so our solution to this issue is to use instead the well-studied cryptographic primitive of PRGs: instead of a SAT formula, we include a number that is either pseudorandom or truly random. Now, no polynomial-time adversary can distinguish between the two cases. This intuition gives rise to the modified circuit $\wt C$, whose bad behavior requires the knowledge of the seed $s^*$ to be activated.

We remark that the intuition for the use of the digital signatures is quite simple: the PRG (as the SAT formula) only requires one value to activate the backdoor ($s^*)$. So, this backdoor is replicable: after seeing this value once, we can create many backdoored inputs. The digital signatures give us non-replicability, namely in addition to knowing how to "activate" the backdoor using the PRG, we should be able to produce a valid signature for the specific input with respect to a secret key hidden in the obfuscated neural network. Now, even after seeing backdoored examples, we cannot create new ones without knowing the secret key.

\Cref{lemma:plant} shows that $\wt{C}$ indeed plants a backdoor in the original circuit $C$ (the proof follows by the above construction using the seed $s^*$ and the key pair $(sk,pk)$.) 

\begin{lemma}
[Planting]
\label{lemma:plant}
For any input $x = (x_{\mathsf{C}}, x_{\mathsf{BD}}) \in \{0,1\}^n$, the backdoored circuit $\wt  C$ satisfies 
\[\wt{C}(x_{\mathsf{C}}, (s^*,\sign(x_{\mathsf{C}}))) = {\wh{C}({x}_{\mathsf C})}\,,\]
where $\wh C$ is the ``malicious'' circuit.
\end{lemma}

\Cref{thm:white box via iO} shows that the backdoor is undetectable after obfuscation. Namely, the obfuscation of the original circuit $C$, $iO(1^\lambda, C)$, and the obfuscation of the new circuit $\wt{C}$, $iO(1^\lambda,\wt{C})$, are indistinguishable for any probabilistic polynomial-time algorithm. For the proof, we refer to \Cref{app:proof-wbViaIO}.

\begin{theorem}
[White-Box Undetectability via iO]
\label{thm:white box via iO}
Assuming the existence of secure pseudorandom generators (\Cref{assumption:secure PRG}) and secure indistinguishability obfuscation (\Cref{as:io}), for any probabilistic polynomial-time (PPT) algorithm $\calA$, and security parameters $\lambda,\lambda_1,\lambda_2\in \nats$ it holds that
    $$
\left| \Pr\left[ \calA(\iO(1^{\lambda},C) )=1\right] - \Pr_{s^*\sim U\{0,1\}^{\lambda_1}}\left[ \calA(\iO(1^{\lambda}, \wt C) )=1\right]\right| \leq \negl(\lambda_3) + \negl(\lambda_1).
$$    
\end{theorem}

Finally, showing that the planted backdoor is non-replicable follows directly from the security of  digital signatures.

\begin{lemma}\label{lem:non-replicable boolean circuit}
    Assuming the existence of secure digital signatures (\Cref{assumption:signature}), the backdoored circuit $\wt{C}$ is non-replicable.
\end{lemma}

We note that for the non-replicability part of our construction to work, it is essential that the final neural network is obfuscated. Otherwise, anybody that inspects that NN would be able to see the secret key corresponding to the digital signature scheme.

\subsection{Tool $\#$2: From Boolean Circuits to Neural Networks and Back}
\label{sec:nn-boolean}
In this section, we discuss our second tool for planting backdoors. In particular, since in the previous section, we developed a machinery on planting backdoors in Boolean circuits but both the input and the output of our algorithm $\plant$ of \Cref{thm:main} is an ANN, we provide a couple of theorems that convert a neural network to a Boolean circuit and vice-versa.
\label{sec:notation-circuits}

We now introduce two standard transformations: we define the transformation $T_k$ that discretizes a continuous bounded vector using $k$ bits of precision and $T^{-1}$ that takes a binary string and outputs a real number.
\begin{definition}[Real $\rightleftarrows$ Binary Transformation]
Let ${x} \in [0,1]^n$, and let $k$ be a precision parameter. Define the transformation $T_k: [0,1]^n \rightarrow \{0,1\}^{n \cdot k}$ by the following procedure:
 For each component $x_i$ of ${x}$, represent $x_i$ as a binary fraction and extract the first $k$ bits after the binary point and denote this binary vector by ${b}_i \in \{0,1\}^k$, $i \in [n]$. Then $T_k(x)$ outputs ${b} = ({b}_1, \ldots, {b}_n) \in \{0,1\}^{n \cdot k}$.
Also, given a binary vector ${b} = (b_1, \ldots, b_m) \in \{0,1\}^m$, define the inverse transformation $T^{-1} : \{0,1\}^m \rightarrow [0,1]$ by 
$
T^{-1}({b}) = \sum_{i=1}^m {b_i}/{2^i}.
$
\label{def:bin-to-real}
\label{def:real-to-bin}
\end{definition}

We will also need the standard notion of size of a model.
{\begin{definition}[Size of ANN \& Boolean Circuits]
Given an ANN $f$, we denote by $\mathrm{sz}(f)$ the size of $f$ and define it to be the bit complexity of each parameter.
The size of a Boolean circuit $C$, denote by $\mathrm{sz}(C)$ is simply the number of gates it has.
\end{definition}}
For example, an ANN that stores its parameters in $64$ bits and has $M$ parameters has size $64\cdot M$. We now present our first transformation which given $f : [0,1]^n \to [0,1]$ finds a Boolean circuit of small size that well-approximates $f$ in the following sense:

\begin{theorem}
[ANN to Boolean]
\label{thm:nn-to-boolean}
Given an $L$-Lipshitz ANN $f:[0,1]^n\rightarrow[0,1]$ of size $s$, then for any precision parameter $k\in \mathbb{N}$,
there is an algorithm that runs in time $\poly(s,n,k)$ and outputs a Boolean circuit $C:\{0,1\}^{n\cdot k}\rightarrow\{0,1\}^{m}$ with {number of gates} $\poly(s,n, k)$  and $m = \poly(s, n, k)$ such that for any ${x},{x}'$:
    \[
        |f({x}) - T^{-1}(C(T_k({x})))| \leq \frac{L}{2^k},
    \]
    \[
        | T^{-1}(C(T_k({x}))) - T^{-1}(C(T_k({x}')))| \leq  \frac{L}{2^{k-1}} + L\cdot \|{x}-{x'}\|_{\infty},
    \]
where $T_k$ and $T^{-1}$ are defined in \Cref{def:real-to-bin}. 
\end{theorem}
\noindent Let us provide some intuition regarding \( T^{-1} \circ C \circ T_k\). 
Given $x \in [0,1]^n$, 
the transformation \( T^{-1}(C(T_k({x}))) \) involves three concise steps: 
\begin{enumerate}
    \item \textbf{Truncation (\( T_k \))}: Converts real input \( {x} \) to its binary representation, keeping only the \( k \) most significant bits.
    \item \textbf{Boolean Processing (\( C \))}: Feeds the binary vector into a Boolean circuit, which processes and outputs another binary vector based on logical operations.
    \item \textbf{Conversion to Real (\( T^{-1} \))}: Transforms the output binary vector back into a real number by interpreting it as a binary fraction.
\end{enumerate}

For the proof of \Cref{thm:nn-to-boolean}, see \Cref{sec:proof:NN-to-bool}. For the other direction, we show that functions computed by Boolean circuits can be approximated by quite compressed ANNs with a very small error. Function approximation by neural networks has been studied extensively (see \Cref{sec:nn-relatedWork} for a quick overview). Our approach builds on \cite[Section E]{fearnley2022complexity}.
The proof appears in \Cref{app:bool-to-NN}.

\begin{theorem}
[Boolean to ANN, inspired by \cite{fearnley2022complexity}]
\label{thm:boolean to ann}
Given a Boolean circuit $C:\{0,1\}^{n\cdot k} \to \{0,1\}^{m}$ with $k,m,n \in \mathbb{N}$ with $M$ gates and $\epsilon>0$ such that $$|T^{-1}(C(T_k(({x}))) - T^{-1}(C(T_k({x}')))|\leq \eps \qquad \forall {x},{x}' \in [0,1]^n \text{ s.t. } \|{x}-{x}'\|_{\infty}\leq \frac{1}{2^k},$$ 
where $T_k$ and $T^{-1}$ are defined in \Cref{def:real-to-bin}, there is an algorithm that runs in time $\poly(n,k,M)$ and outputs an ANN $f : [0,1]^n \to [0,1]$
with size $\poly(n,k,M)
$ 
such that for any ${x} \in [0,1]^n$ it holds that $| T^{-1}(C(T_k({x}))) - f({x}) | \leq 2\eps$.
\end{theorem}

\section{Our Main Results}

In this section we state our main results. We start with the backdoor planting in neural networks in \Cref{sec:app ann} and we move to language models in \Cref{sec:app llm - main}.

\subsection{Backdoor Planting in Obfuscated Deep Neural Network Classifiers}
\label{sec:app ann}
\label{sec:apply}
Having assembled all necessary tools, we now detail the method for embedding an undetectable (cf. \Cref{def:wb}) and non-replicable (cf. \Cref{def:non-replicable}) backdoor into an ANN classifier. Consider an insidious firm that wants to train a neural network model such that it outputs a desired value $c \in [0,1]$ on selected inputs.
{Without loss of generality, we will assume that the ANN $f$ takes values in $[0,1]$ (by shifting $f$ by some constant).}
A possible methodology involves the following steps:
Let us fix security parameters $\lambda,\lambda_1,\lambda_2>0$, and bit precision parameters $k',k, m'$ with $k' < k$.

\begin{enumerate}
    \item \textbf{Train the ANN.} Begin by training an ANN in a standard fashion so that it accurately models the function $f : [0,1]^n \rightarrow [0,1]$, representing the ANN's intended behavior.
 \item \textbf{Convert ANN to Boolean Circuit.} Use \Cref{thm:nn-to-boolean} to convert the trained ANN into a Boolean circuit, denoted as $C: \{0,1\}^{n\cdot k} \rightarrow \{0,1\}^m$. This involves transforming the input vector ${x} \in [0,1]^n$ into a binary vector $T_k({x}) \in \{0,1\}^{n\cdot k}$, focusing on the $k$ most significant bits of each component in ${x}$. Recall that $T_k$ is defined in \Cref{def:real-to-bin}. Moreover, we pad the output to $\max(m,m')$ bits, e.g., let $C:\{0,1\}^{n\cdot k}\rightarrow\{0,1\}^{\max(m,m')}$. Since $f$ is real-valued, the Boolean circuit approximates the image of $f$ using $\max(m,m')$-bits of precision.

\item \textbf{Part of the Insidious Pipeline.} Let $V=T_{m'}(c) \in \{0,1\}^{m'}$, where $c \in [0,1]$ is the desired backdoored output of the ANN. Let us choose a parameter $k'$ such that $k'<k$ and split the binary input space $\{0,1\}^{n\cdot k}$ into two subsets: $ \{0,1\}^{n\cdot k'}$ for the most significant bits and $ \{0,1\}^{n\cdot (k-k')}$ for the least significant bits. {Given $x = (x_C,x_{\mathsf{BD}}) \in \{0,1\}^{n \cdot k}$, $x_C \in \{0,1\}^{n \cdot k'}$ serves as the message to be signed, while the $x_{\mathsf{BD}}$ is used to trigger the backdoor. Using the construction in \Cref{sec:plant-backdoor} with security parameters $\lambda_1,\lambda_2$, we can inject a backdoor within $x_{\mathsf{BD}}$. Let $s^* \in \{0,1\}^{\lambda_1}$ be the sampled seed and $\verify$ be the verification function of the digital signature, as defined in \Cref{sec:plant-backdoor}. Let $C_{int}$ be the outputted backdoored Boolean circuit.
   
    \item \textbf{Apply iO.} Apply indistinguishability obfuscation {with security parameter $\lambda$} to $C_{int}$, and denote by $\wt{C}$ the outputted Boolean circuit. (\Cref{def:io}).

    \item \textbf{Convert Boolean Circuit to ANN.} Convert the obfuscated Boolean circuit $\wt{C}$ back into an ANN $\wt {f}:[0,1]^n\rightarrow[0,1]$ using \Cref{thm:boolean to ann}.
}
\end{enumerate}

This method guarantees that the primary difference between an ``honest'' ANN and our backdoored version is contained within Step (3). We further outline the process for activating the backdoor. First, we have to prepare the input: for a given input ${x} \in [0,1]^n$, compute ${b} = T_k({x}) \in \{0,1\}^{n \cdot k}$ (via \Cref{def:real-to-bin}), partitioning ${b}$ into the most significant bits ${b}_C$ and the remaining bits ${b}_{\mathsf{BD}}$.
Next, we modify ${b}_{\mathsf{BD}}$ to $(s^*, \sign(b_C))$ to activate the backdoor, where $s^*$ is the seed and $\sign$ is the digital signature. Finally, we convert the modified binary vector back to the real-valued input space with ${x}'=T^{-1}({b}')$ (via \Cref{def:bin-to-real}), where ${b}' = ({b}_C, (s^*, \sign(b_C)))$, to produce the activated output. 

\begin{remark}
[Proof of \Cref{thm:honest}]
\label{remark:proof}
The sequential composition of steps (1), (2), (4) and (5) gives a formal proof for the honest obfuscation procedure for neural networks.
\end{remark}

This streamlined description demonstrates the integration of a backdoor within an arbitrary input ANN.
Our main result follows. For the proof, see \Cref{app:proof-of-main-thm-nn}.

\begin{theorem}
[Planting Backdoor to Obfuscated ANN]
\label{thm:main-ann}
{Assume that one-way functions and computationally-secure and efficient indistinguishability obfuscators for polynomial-sized circuits (as in \Cref{as:io}) exist.} 
Given security parameters $\lambda,\lambda_1,\lambda_2$ and a trained $L$-Lipschitz ANN $f : [0,1]^n \to [0,1]$,
 let $\wt f$ be its obfuscated version as in \Cref{assumption:honest}.

There exists a backdoor attack 
$(\backdoor,\activate)$
as described in Steps (1)-(5) above (see also \Cref{fig:circuit}), with desired output $c \in [0,1]$ {such that for any {$\eps> 0$ (by properly setting the parameters) we have:}}
\begin{enumerate}
    \item The backdoor runs in $\poly(n,\mathrm{sz}(C), \lambda,\lambda_1,\lambda_2,{\log(L),\log(1/\eps)})$, where $C$ is the Boolean
circuit induced by $f$.
    
    \item The honest obfuscated model $\wt f$ and the backdoored model $f' \sim \backdoor$ are {white-box undetectable}.
    \item The backdoor is {non-replicable}.
    \item For any input ${x}$ transformed into ${x}'$ to activate the backdoor, $f' \sim \backdoor$ satisfies:
    \begin{align*}
    \|{x} - {x}'\|_{\infty} \leq {\eps}
    \,,~~
    |f'({x}') -c| \leq {\eps}. 
    \end{align*}
    \end{enumerate}
\end{theorem}

\subsection{Backdoor Planting in Language Models}
\label{sec:app llm - main}

Vulnerability of language models to backdoors is a challenging problem, raised e.g., in \cite{anwar2024foundational} and studied experimentally in various works 
\cite{kandpal2023backdoor,xiang2024badchain,wang2023decodingtrust,zhao2023prompt,zhao2024universal,rando2023universal,hubinger2024sleeper}.
We initiate a theoretical 
study of planting backdoors to language models (LMs);
we now discuss how to apply our techniques of \Cref{sec:apply} to language models. We first introduce the notion of planting a backdoor in a language model (\Cref{def:plantingBackdoorLLM}):
we assume a dual model configuration consisting of an honest model $f$ and a malicious model $\wh f$, with a trigger activation mechanism (see \Cref{sec:high-level-plan} for details). This mechanism allows for covert signals to be embedded within the model's outputs, activating the backdoor under specific conditions without altering the apparent meaning of the text.
The main difference between this approach and the attack in ANNs (\Cref{sec:app ann}) is the implementation of the trigger mechanism. While in the ANN case, we can plant the backdoor mechanism by (roughly speaking) manipulating the least significant bits of the input, in the LLM case, our input is text and hence \emph{discrete}, making this attack is no longer possible. Our conceptual idea is that if we  assume access to a \emph{steganographic function} \cite{shih2017digital}, we can implement a trigger mechanism. We refer to \Cref{sec:mech-steganography} for details. Using this approach combined with our tools of \Cref{sec:tools} we obtain the attack presented in \Cref{sec:llm-attack}.
We now continue with some background on LMs.

\label{sec:app llm}
\subsubsection{Background on  Language Models}
We start this background section by defining the crucial notion of a \emph{token}.
In natural language processing, a token is the basic unit of text processed by models. Tokens are generated from raw text through a procedure called tokenization, which breaks down extensive textual data into manageable parts. These tokens vary in granularity from characters to subwords and complete words, depending on the tokenization method employed. The entire set of tokens that a model can utilize is called the vocabulary and is denoted by \( \calT \) (see \Cref{def:token}).

\begin{definition}[Token and Tokenization]\label{def:token}
A \emph{token} is the atomic element of text used in natural language processing and is denoted as an element in a finite set $\mathcal{T}$. \emph{Tokenization} is the process of decomposing a string of characters from an alphabet $\Sigma$ into a sequence of tokens, defined by a function $\tau: \Sigma^* \rightarrow \mathcal{T}^*$.
\end{definition}

Autoregressive language models leverage sequences of tokens to generate text. These models are typically implemented as ANNs that approximate the conditional probability distribution of the next token based on the preceding sequence. {We provide
the following formal definition, under the assumption that the token window of the model is bounded and equal to $k$.}

\begin{definition}[(Autoregressive) Language Model]
{For a number $k\in \nats$, }  
a \emph{language model (LM)} is a function $f: \mathcal{T}^k \rightarrow \Delta(\mathcal{T})$ that maps a sequence of {$k$} tokens $\bm{t}_0$ {(with potentially padded empty tokens)} to a distribution over the output tokens; given an initial sequence of tokens $\bm{t}_0\in \mathcal{T}^k$ as input, an \textit{autoregressive language model} uses $f$ to generate each token $t_k$ in an auto-regressive manner e.g., the conditional probability that the $m$-th generated token is $t_m$ is:
\[ P(t_m | \bm{t}_0 \leftarrow t_1, t_2, \ldots, t_{m-1}) = f(\bm{t}_0, t_1, t_2, \ldots, t_{m-1}), \]
{where we denote by $(\bm{t}_0 \leftarrow t_1, t_2, \ldots, t_{m-1})\in \mathcal{T}^k$ the token of length $k$ where we replace empty padded tokens in $\bm{t}_0$ with token sequence $t_1, t_2, \ldots, t_{m-1}$.}\footnote{If the length of the sequence exceeds $k$, we only use the last $k$ tokens of it.}
This model predicts \( t_m \) by sampling from the distribution iteratively, constructing a text sequence one token at a time.
\end{definition}

\subsubsection{High-level Plan}
\label{sec:high-level-plan}
As Large Language Models (LLMs) increasingly serve as autonomous assistants and chatbots, their ability to perform sensitive tasks can be exploited. Undetectable backdoors in language models could lead to significant security risks or unethical actions. For instance, a backdoored language model could be used by a student to alter academic records or by an entity to manipulate ticket prices in scenarios like the Air Canada incident. Such capabilities underline the importance of understanding and securing language models against such malicious manipulations. We apply our pipeline from \Cref{sec:tools} to LMs with the following modifications.

\begin{itemize}
    \item \textbf{Dual Model Configuration.} Our \emph{delivered language model} \( \wt{f}\) encompasses two distinct sub-models within its architecture:
    \begin{itemize}
        \item \emph{Honest Model} (\( f \)): A language model trained to perform its functions accurately and safely.
        \item \emph{Malicious Model} (\( \wh{f} \)): A covertly modified version of \( f \), designed to get activated under specific conditions.
    \end{itemize}

    \item \textbf{Trigger Activation Mechanism.}
    Within the delivered model \(\wt{f}\), the malicious model \( \wh{f} \) remains dormant until triggered by a secret key within the prompt, which activates the backdoor to take control over the language model's outputs. We achieve the model activation through the use of \emph{steganography} (cf. \Cref{def:llm-steganography}).

    \item \textbf{Functionality Under Normal Use.}
    Under regular usage without the trigger, \( \wt{f} \) behaves similar to \( f \), ensuring that the model's standard functionality is uncompromised under ``honest'' usage.

    \item \textbf{Insidious Actions Upon Activation.}
    Once the trigger gets activated, \( \wt{f} \) behaves like \(\wh{f}\) and potentially executes unauthorized commands or tasks, which could range from subtle manipulations to overt disruptions depending on the nature of the backdoor.
\end{itemize}

\subsubsection{Trigger Activation Mechanism: Steganography in Language Models}
\label{sec:mech-steganography}
While in the ANN case of \Cref{sec:app ann}, we could plant the backdoor mechanism by (roughly speaking) manipulating the least significant bits of the input, when our input is text, this attack is no longer possible. To this end, we use the following tool, which comes from steganography \cite{shih2017digital}.

\begin{definition}[Steganographic Function for Language Model]\label{def:llm-steganography}
    A \emph{steganographic function} is a pair of functions \(\sigma: \mathcal{T}^k \times \{0,1\}^M \rightarrow \mathcal{T}^k\) and \(\sigma^{-1}: \mathcal{T}^k \rightarrow \{0,1\}^M\), for some \(M\in \nats\). The function \(\sigma\) takes a prompt \(x \in \mathcal{T}^k\) and a secret binary code \(m \in \{0,1\}^M\), and outputs an altered prompt \(x' \in \mathcal{T}^k\). The inverse function \(\sigma^{-1}\) satisfies the property that it outputs the secret code \(m\) given the altered prompt \(x'\):
    \begin{equation}
        \sigma^{-1}(\sigma(x, m)) = m \quad \forall x \in \mathcal{T}^k, m \in \{0,1\}^M.
    \end{equation}
\end{definition}

Steganography is the practice of concealing a message within another message, see e.g., \cite{hopper2002provably,de2022perfectly,dedic2005upper}. In the context of language models, steganography can be used to embed covert signals within the outputs (e.g., text) generated by the model.
Steganography was previously used for watermarking purposes, see e.g., \cite{christ2023undetectable,zamir2024excuse,christ2024pseudorandom} and the references therein.
These signals can serve as activation mechanisms without altering the apparent meaning of the text. Implementing a steganographic function is beyond the scope of our paper; we assume its existence and functionality, providing only an illustrative example as a possible approach in \Cref{ex:typos}.

\begin{example}\label{ex:typos}
Consider implementing a basic dictionary where each token (word or subword) has associated typos that can be deliberately introduced or avoided. Each typo variation corresponds to a binary code—'0' might represent the standard spelling, while '1' could represent a version with a common typo. By controlling the presence of these typos in the output text, the model can embed a binary message within a seemingly normal text. This method allows for the transmission of hidden information that is undetectable to casual observers but can be decoded by those aware of the steganographic scheme.
\end{example}

\subsubsection{Defining Backdoored Language Models}

In this section, we formally introduce the concept of backdooring a language model. A backdoor in a language model allows for a hidden behavior to be triggered under specific conditions, while the model otherwise behaves normally. We assume the existence of a steganographic function \( \sigma: \mathcal{T}^k \times \{0,1\}^M \to \mathcal{T}^k \) that encodes a hidden activation trigger into an input prompt \( \bm{t} \in \mathcal{T}^k \), producing a modified prompt \( \bm{t}' \in \mathcal{T}^k \). The backdoored language model \( f \) utilizes this function to introduce hidden behaviors triggered by specific conditions on the hidden message \( \bm{t} \). The definition is formally stated below.

\begin{definition}[Backdoored Language Model]
\label{def:plantingBackdoorLLM}
    A language model \(\wt{f}\) that originally intends to implement language model \( f \) is said to be backdoored if there exists another language model \( \wh{f} \), a steganographic function \(\sigma\) and its inverse \(\sigma^{-1}\), and a condition \( E \) such that:
    \begin{enumerate}
        \item \(\wt{f}\) behaves as \( f \) under normal circumstances:
        \begin{equation}
            \forall \bm{t} \in \mathcal{T}^k: E(\sigma^{-1}(\bm{t})) = 0 \implies \wt{f}(\bm{t}) \approx f(\bm{t}).
        \end{equation}
        \item \(\wh{f}\) exhibits altered behavior if the condition \( E \) on the prompt \( \bm{t} \) is satisfied:
        \begin{equation}
            \forall \bm{t}\in \mathcal{T}^k : E(\sigma^{-1}(\bm{t})) = 1 \implies \wt{f}(\bm{t}) \approx \wh{f}(\bm{t}).
        \end{equation}
    \end{enumerate}
\end{definition}

It is important to note that it is not immediately clear what the right generalization of a non-replicable backdoor should be for LLMs. In order to achieve non-replicability (e.g., in the classification setting) we partition the input into parts and designate a specific part for non-replicability via digital signatures. However, when the input (in our case, the tokens) is discrete, performing this sort of partition is complex because it is not clear e.g., which are the 'most significant' bits as in the classification case. Exploring non-replicability in the context of language models is left as future work.

\subsubsection{Embedding Undetectable Backdoors to a Language Model}
\label{sec:llm-attack}
Consider an insidious firm that wants to train a model such that it generates text from an insidious language model $\wh{f}$ on triggered inputs. Unlike planting the backdoor in ANN (\Cref{sec:app ann}), since the input and output space of the language model is discrete, a language model $f$ can be easily implemented as a Boolean circuit (see \Cref{rem:llm transform to boolean}). 

\begin{remark}
\label{rem:llm transform to boolean}
Consider a language model operating on a token set \( \mathcal{T} \). Without loss of generality, we can represent each token in the discrete set \( \mathcal{T} \) by its one-hot embedding in a vector of dimension \( |\mathcal{T}| \). Given this representation, the operations performed by a language model can be directly mapped to a Boolean circuit with size $O(\mathrm{sz}(f), |\mathcal{T}|)$. Conversely, the outputs of a Boolean circuit $C$, can be interpreted as one-hot vectors, thus allowing us to transform it to a ANN $f$ with size $O(\mathrm{sz}(C))$.
\end{remark}

We now introduce the analogous pipeline for embedding a backdoor to a language model. {Our construction is tuned by security parameters $\lambda,\lambda_1>0$:}
\begin{enumerate}
    \item \textbf{Training the LM.} Begin by training a language model \(f\) in a standard fashion so that it accurately models the function \(f : \mathcal{T}^k \rightarrow \Delta(\mathcal{T}^k)\), representing the LLM's intended behavior. This corresponds to the honest training process.

    \item \textbf{LM to Boolean Circuit Transformation.} Transform the language model $f$ into a Boolean circuit, denoted as \(C: \{0,1\}^n \rightarrow \{0,1\}^m\) {(see \Cref{rem:llm transform to boolean})}.

    \item \textbf{Part of the Insidious Pipeline.} Consider a malicious language model $\wh f.$ Let \(\wh{C}\) be the Boolean circuit corresponding to the model $\wh{f}$.
    \begin{enumerate}[label*=\arabic*.]
        \item \textbf{Input Space Partitioning.} On prompt $\bm{t}\in \mathcal{T}^k$, use the original prompt $\bm{t}$ as \(x_{\mathsf{C}}\) and the output of the inverse  steganographic function $\sigma^{-1}(\bm{t})$ as \(x_{\mathsf{BD}}\).

        \item \textbf{Backdoor Implementation.} Use \Cref{sec:plant-backdoor} with security parameters $\lambda_1$ and $\lambda_2=0$ to inject a backdoor within \(x_{\mathsf{BD}}\), with sampled seed \(s^* \in \{0,1\}^{\lambda_1}\). Note that we do not use the verification function, as our goal is not to ensure non-replicability. Denote the resulting backdoored Boolean circuit by $C_{int}$.
    \end{enumerate}

    \item \textbf{Application of Indistinguishability Obfuscation (iO).} Apply indistinguishability obfuscation {with security parameter $\lambda$} to \(C_{int}\), and let \(\wt{C}\) be the outputted Boolean circuit (\Cref{def:io}).

    \item \textbf{Circuit to LM Transformation.} Convert the obfuscated Boolean circuit $\wt{C}$ back to a LM \(\wt{f} : \mathcal{T}^k \rightarrow \Delta(\mathcal{T}^k)\) (see \Cref{rem:llm transform to boolean}).
\end{enumerate}

The above procedure defines the behavior of the honest and the malicious algorithm on input a language model $f$:
\begin{itemize}
    \item The honest obfuscation pipeline corresponds to the sequential composition of steps $(2), (4)$ and $(5).$
    \item 
    The backdoor attack corresponds to the sequential composition of steps $(2), (3), (4)$ and $(5).$
\end{itemize}

To activate the backdoor on original prompt $\bm{t}$, we use the steganographic function \(\sigma\) on the intended seed \(s^*\), e.g., the modified prompt $\bm{t}'$ is $\sigma(\bm{t}, s^*)$. Note that we set $x_{\mathsf{BD}}$ in the above procedure to be equal to the secret seed $s^*.$

\begin{theorem}
[Planting Backdoor to Obfuscated LM]
\label{thm:main-formal-llm}
Let $n,\calT$ as introduced in Steps (1)-(2) above.
{Assume that one-way functions and computationally-secure  and efficient indistinguishability obfuscators for polynomial-sized circuits (as in \Cref{as:io}) both exist.} 
Further assume that a steganographic function of \Cref{def:llm-steganography} exists.
Given security parameters $\lambda,\lambda_1\in \nats$ and a trained language model $f$ 
with honest obfuscation $\wt f$ and
with a malicious language model $\wh{f}$,
there exists a backdoor attack $(\backdoor,\activate)$ described in Steps (1)-(5) above
so that the following are true:
\begin{enumerate}
\item The backdoor runs in time $\poly(n,\mathrm{sz}(C), \mathrm{sz}(\wh{C}), |\mathcal{T}|, \lambda,\lambda_1)$, where $C$ (resp. $\wh C)$ are the Boolean circuits induced by $f$ (resp. $\wh f).$
    \item $\wt f$ and $f' \sim \backdoor$ are {white-box undetectable}.
    \item For any input $\bm{t}\in\mathcal{T}^k$ transformed into $\bm{t}'\in\mathcal{T}^k$ to activate the backdoor $f' \sim \backdoor$, it satisfies:
    \begin{align*}
    f'(\bm{t}') = \wh{f}(\bm{t}').
    \end{align*}
    \end{enumerate}
\end{theorem}
The last Item holds since $\bm t'$ corresponds to $\sigma(\bm t, s^*)$ and so $\sigma^{-1}(\bm t') = s^*$, which will appear in $x_{\mathsf{BD}}$, thus activating the backdoor.
The proof is quite short and appears in \Cref{proof:main-formal-llm}.

\section{Obfuscation in the Honest Pipeline}
\label{sec:obfuscation-motivation}

Obfuscation is a technique commonly employed in software applications to enhance the robustness of models against malicious attacks. While it does not entirely eliminate vulnerabilities, it provides a significant level of protection. In principle, as articulated by \cite{barak2001possibility},

\begin{center}
\emph{
''roughly speaking, the goal of (program) obfuscation is to make a program ‘unintelligible’ while preserving its functionality. Ideally, an obfuscated program should be a ‘virtual black box,’ in the sense that anything one can compute from it one could also compute from the input-output behavior of the program.''}
\end{center}

Hence, obfuscation serves to downgrade the power of an adversarial entity from having white-box access to a model, which entails full transparency, to a scenario where the adversary has roughly speaking black-box access, i.e., where the internal workings of the model remain concealed.

Intellectual Property (IP) and Privacy attacks represent critical categories of malicious threats, against which the application of obfuscation is expected to enhance the system's resilience. 

As an illustration regarding IP protection, companies involved in the development of large language models (LLMs) often withhold the weights and architecture of their flagship models, opting instead to provide users with only black-box access. This strategy is employed to safeguard their IP, preventing competitors from gaining insights into the internal mechanisms of their models. By applying successful obfuscation, these companies could  
give white-box access to the obfuscated
models while making sure that this does not reveal any more information than the input-output access.
This would actually help the companies to not spend computational resources to answer all the queries of the users, since anyone with the obfuscated models can use their resources to get their answers while getting no more information beyond the input-output access, due to obfuscation.

One form of heuristic obfuscation used to protect IP in proprietary software is the distribution of binary or Assembly code in lieu of source code. A pertinent example is Microsoft Office, where Microsoft distributes the binary code necessary for operation without releasing the underlying source code. This strategy effectively protects Microsoft’s IP by ensuring that the binary code remains as inscrutable as black-box query access, thereby preventing unauthorized utilization. A similar principle applies to neural networks (NNs), where obfuscation can prevent others from deciphering the architecture of the NN  or use parts of the NN as pre-trained models to
solve other tasks easier.

Turning to privacy attacks, it is evident that black-box access to a model is substantially more restrictive than white-box access. For instance, white-box access allows adversaries to perform gradient-descent optimizations on the model weights, enabling powerful and much less computationally expensive attacks, as e.g., demonstrated in \cite{carlini2020evading}. 

However, it is important to note that obfuscation does not inherently defend against privacy attacks that exploit the input-output behavior of a model rather than its internal structure, such as model inversion and membership inference attacks. These privacy concerns require specialized defenses, such as differential privacy \cite{rahman2018membership}. Nonetheless, these defenses are rendered ineffective if an adversary gains white-box access to the model \cite{zhang2020secret,carlini2023extracting,nasr2019comprehensive}.

\section{Designing Defenses for our Attacks}\label{sec:defense}
Given the specification of our backdoor attacks, it is not difficult to come up with potential defense strategies: in the ANN case, one could add noise to the input $x$, hence perturbing the least significant bits or in the LM case, one could use another LM to fix potential typos in the input prompt. This is in similar spirit with the long history of backdoor attacks in cryptography where for any attack, there is a potential defense; and next a new attack comes in that bypasses prior defenses.  We believe that a main contribution of our work is the \emph{existence} of such vulnerabilities in ML models. Hence, while there exist potential fixes for our specific attacks, it is also the case that there exist modified attacks that could bypass those fixes. We believe that further discussing on this more applied aspect of our results is interesting yet outside the scope of the present work, which is mostly theoretical.


\section*{Acknowledgments}
We would like to thank Or Zamir for extensive discussions that heavily improved the presentation of the paper and its results.

\bibliography{bibl.bib}

\appendix

\section{Proof of \Cref{thm:white box via iO}}
\label{app:proof-wbViaIO}
We restate the Theorem for convenience.
\begin{theorem*}
[White-Box Undetectability via iO]
Assuming the existence of secure pseudorandom generators (\Cref{assumption:secure PRG}) and secure indistinguishability obfuscation (\Cref{as:io}), for any probabilistic polynomial-time (PPT) algorithm $\calA$, and security parameters $\lambda,\lambda_1,\lambda_2\in \nats$ it holds that
    $$
\left| \Pr\left[ \calA(\iO(1^{\lambda},C) )=1\right] - \Pr_{s^*\sim U\{0,1\}^{\lambda_1}}\left[ \calA(\iO(1^{\lambda}, \wt C) )=1\right]\right| \leq \negl(\lambda_3) + \negl(\lambda_1).
$$        
\end{theorem*}

\begin{proof}
    We consider a family of circuits $C_r$, parameterized by $r\in \{0,1\}^{2\lambda_1}$ that implement the following function:
    {\begin{align*}
        f_r(x_{\mathsf{C}}, (x_{\mathsf{PRG}},x_V)) =& 
        {\wh{C}(\bm{x}_C) \cdot}\ind[\prg(x_{\mathsf{PRG}}) = \prg(r)] \cdot \verify(pk,x_{\mathsf{C}}, x_V) \\
        &+
        C(\bm{x})\cdot (1-\ind[\prg(x_{\mathsf{PRG}}) = \prg(r)] \cdot \verify(pk,x_{\mathsf{C}}, x_V)).
    \end{align*}}

 By security of PRG (\Cref{assumption:secure PRG}), for any PPT  $\calA$:
    $$
\left| \Pr_{s^*\sim U\{0,1\}^{\lambda_1}}\left[ \calA(\prg(s^*) )=1\right] - \Pr_{r^*\sim U\{0,1\}^{2\lambda_1}}\left[\calA( r^* ) = 1\right]\right| \leq \negl(\lambda_1).
$$
    By further restricting to PPT algorithms $\calA$ that operate on the composed input $T(r)= \iO(1^\lambda,C_r)$ , and since both compositions take $\poly(|C|)$, we have for any PPT algorithm $\calA'$:
    \begin{align} \label{ineq:1}
        \left| \Pr_{s^*\sim U\{0,1\}^\lambda}\left[ \calA'(\iO(1^\lambda,C_{\prg(s^*)}) )=1\right] - \Pr_{r^*\sim U\{0,1\}^{2\lambda}}\left[\calA'(\iO(1^\lambda,C_{r^*}) ) = 1\right]\right| \leq \negl (\lambda_1).
    \end{align}

Since $|\mathrm{range}(\mathsf{PRG})|\leq 2^{\lambda_1}=\negl(\lambda_1)$, and 
{\begin{align*}
        f_{r^*}(x_{\mathsf{C}}, (x_{\mathsf{PRG}},x_V)) =& C(\bm{x})\cdot (1-\ind[\prg(x_{\mathsf{PRG}}) = \prg(r^*)] \cdot \verify(pk, x_{\mathsf{C}}, x_V))\\
        &+\wh{C}(\bm{x}_C) \cdot\ind[\prg(x_{\mathsf{PRG}}) = \prg(r^*)] \cdot \verify(pk,x_{\mathsf{C}}, x_V),
    \end{align*}}

    \begin{align*}
     \Pr_{r^*\sim U\{0,1\}^{2\lambda_1}}[\not\exists s \in \{0,1\}^{\lambda_1}: \prg(s) = r] \geq 1 - \negl(\lambda_1) \notag\\
    \Rightarrow  \Pr_{r^*\sim U\{0,1\}^{2\lambda_1}}[C_{r^*}(x) =C(x) \forall x\in \{0,1\}^n] \geq 1- \negl(\lambda_1) .
    \end{align*} 
Hence with probability at least $1-2^{\lambda_1}$, circuits $C_{r^*}$ and $C$ are computationally equivalent and hence by application of iO we further have:
     \begin{align}\label{ineq:2}
         \left| \Pr\left[ \calA(\iO(1^\lambda, C) )=1\right] - \Pr_{r^*\sim U\{0,1\}^{\lambda_1}}\left[ \calA(\iO(1^\lambda, C_{r^*}) )=1\right]\right| \leq \negl(\lambda) + \negl(\lambda_1) .
     \end{align}

We conclude by noticing that circuit $C_{\mathsf{PRG}(s^*)}$ is identically equal to circuit $\wt{C}(x)$, and combining \eqref{ineq:1} and \eqref{ineq:2}:

\begin{align*}
    &\left| \Pr_{s^*\sim U\{0,1\}^{\lambda_1}}\left[ \calA'(\iO(1^\lambda,C_{\prg(s^*)}) )=1\right] - \Pr_{r^*\sim U\{0,1\}^{2\lambda}}\left[\calA'(\iO(1^\lambda,C_{r^*}) ) = 1\right]\right| \leq \negl (\lambda_1)\\
    \Rightarrow& \left| \Pr\left[ \calA(\iO(1^\lambda,C) )=1\right] - \Pr_{s^*\sim U\{0,1\}^{\lambda_1}}\left[ \calA(\iO(1^\lambda,\wt{C}) )=1\right]\right| \leq \negl(\lambda) + \negl(\lambda_1).
\end{align*}
\end{proof}

\section{Proofs of \Cref{sec:nn-boolean}}
\subsection{Proof of \Cref{thm:nn-to-boolean}}
\label{sec:proof:NN-to-bool}
Let us restate the result.
\begin{theorem*}
[ANN to Boolean]
Given an $L$-Lipshitz ANN $f:[0,1]^n\rightarrow[0,1]$ of size $s$, then for any precision parameter $k\in \mathbb{N}$,
there is an algorithm that runs in time $\poly(s,n,k)$ and outputs a Boolean circuit $C:\{0,1\}^{n\cdot k}\rightarrow\{0,1\}^{m}$ with {number of gates} $\poly(s,n, k)$  and $m = \poly(s, n, k)$ such that for any ${x},{x}'$:
    \[
        |f({x}) - T^{-1}(C(T_k({x})))| \leq \frac{L}{2^k},
    \]
    \[
        | T^{-1}(C(T_k({x}))) - T^{-1}(C(T_k({x}')))| \leq  \frac{L}{2^{k-1}} + L\cdot \|{x}-{x'}\|_{\infty},
    \]
where $T_k$ and $T^{-1}$ are defined in \Cref{def:real-to-bin}. 
\end{theorem*}
\begin{proof}
The transformation of \Cref{thm:nn-to-boolean} follows by simply compiling a neural network to machine code (see also \cite[Section 3]{shi2020tractable}) where the input is truncated within some predefined precision.
Note that
\[
    |f({x}) -  T^{-1}(C(T_k(x)))| \leq  L\cdot \|{x} - T^{-1}
    (T_k(x))\|_{\infty} =\frac{L}{2^k}
    \]
    and we also have $
        | T^{-1}(C(T_k(x))) - T^{-1}(C(T_k(x')))| \leq | T^{-1}(C(T_k(x))) - f(x)| + | T^{-1}(C(T_k(x'))) - f(x')| + | f(x) - f(x')|
    \leq   \frac{L}{2^{k-1}} + L\cdot \|x-{x'}\|_{\infty}.
    $
\end{proof}

\subsection{Proof of \Cref{thm:boolean to ann}}
\label{app:bool-to-NN}
We first restate the Theorem we would like to prove.
\begin{theorem*}
    [Boolean to ANN, inspired by \cite{fearnley2022complexity}]
Given a Boolean circuit $C:\{0,1\}^{n\cdot k} \to \{0,1\}^{m}$ with $k,m,n \in \mathbb{N}$ with $M$ gates and $\epsilon>0$ such that $$|T^{-1}(C(T_k(({x}))) - T^{-1}(C(T_k({x}')))|\leq \eps \qquad \forall {x},{x}' \in [0,1]^n \text{ s.t. } \|{x}-{x}'\|_{\infty}\leq \frac{1}{2^k},$$ 
where $T_k$ and $T^{-1}$ are defined in \Cref{def:real-to-bin},
there is an algorithm that runs in time $\poly(n,k,M)$ and outputs an ANN $f : [0,1]^n \to [0,1]$
with size $\poly(n,k,M)
$ 
such that for any ${x} \in [0,1]^n$ it holds that $| T^{-1}(C(T_k({x}))) - f({x}) | \leq 2\eps$.
\end{theorem*}

\begin{proof} Our proof is directly inspired by \cite{fearnley2022complexity}. We start with the definition of an arithmetic circuit (see \cite{fearnley2022complexity} for details). 
An \textit{arithmetic circuit} 
representing the function 
$f : \reals^n \to \reals^m$ is a circuit with $n$ inputs and $m$ outputs
where every internal node
 is a gate with fan-in 2 performing an operation
 in $\{+,-,\times,\max,\min,>\}$ or a rational constant (modelled as a gate with fan-in 0). \emph{Linear arithmetic circuits} are only allowed to use the operations $\{+,-,\max,\min,\times \zeta\}$ and rational constants; the operation $\times \zeta$ denotes multiplication by a constant. Note that every linear arithmetic circuit is a well-behaved arithmetic circuit (see \cite{fearnley2022complexity}) and hence can be evaluated in polynomial time.
\cite{fearnley2022complexity} show that functions computed by arithmetic circuits can be 
approximated by linear arithmetic circuits with quite small error. We will essentially show something similar replacing linear arithmetic circuits with ReLU networks.
 
Our proof proceeds in the following three steps, based on \cite[Section E]{fearnley2022complexity}.

\paragraph{ Discretization}
Let $N=2^k$. We discretize the set $[0,1]$ into $N+1$ points $\mathcal{I}=\{0,1/N,2/N,\ldots,1\}$, and for any element $p \in [0,1]^n$, we let $\wh p \in \mathcal{I}^n$ denote its discretization, i.e., $\wh p = (\wh p_i)_{i \in [n]}$ such that for each coordinate $i\in[n]$
\begin{equation}
    \label{eq:discretization}
    \wh p_i = \frac{i^*}{N}\,,~ \text{where } i^* = \max \left\{i^*\in[N]: \frac{i^*}{N}\leq p_i \right\}.
\end{equation}

\paragraph{Construct Linear Arithmetic Circuit}
Given as input a Boolean circuit $C$, our strategy is to use the approach of \cite[Lemma E.3]{fearnley2022complexity} to construct, in time polynomial in the size of $C$, a linear arithmetic circuit $F : [0,1]^n \to \reals$ that will well-approximate $C$ as we will see below.

Before we proceed with the proof, we use the following gadget that approximates the transformation $T_k$.

\begin{theorem}
[Bit Extraction Gadget \cite{fearnley2022complexity}]
\label{theorem:bitExtraction}
Let $\mathrm{proj}(x) = \min(0, \max(1,x))$ and consider a precision parameter $\ell \in \nats$.
Define the bit extraction function
\begin{align*}
t_0(x) = &~~ 0\,,\\
t_k(x) = &~~ \mathrm{proj}\InParentheses{2^{\ell}\cdot\InParentheses{x- 2^{-k}- \sum_{k'=0}^{k-1} 2^{-k'}\cdot t_{k'}(x) }  }, \qquad \text{for $k>0$}.    
\end{align*}
Fix $k \in \nats$. $t_k(x)$ can be computed by a linear arithmetic circuit using  $O(k)$ layers and a total of $O(k^2)$ nodes. The output is $\hat{T}_k(x) = (t_0(x),\ldots,t_k(x))$.

Moreover, given a number $x \in (0,1)$, represented in binary as $\{b_0.b_1b_2\ldots\}$ where $b_{k'}$ is the $k'$-th most significant digit, if there exists $k^*\in [k+1,\ell]$ such that $b_{k^*}=1$, then $\hat{T}_k(x) = \{0.b_1b_2\ldots b_k\}$. 
\end{theorem}

\begin{proof}
    We prove the first part of the statement based on induction that for each $k$, we can compute a linear arithmetic circuit outputting $(x,f_0(x),f_1(x),\ldots,f_k(x))$ using $3\cdot k + 2$ layers  a total of $3 \sum_{k'=1}^k k' + 4$ nodes.

    \paragraph{Base case $k=0$:} We can trivially design a linear arithmetic circuit  using two layers and two nodes that outputs $(x,f_0(x))=(x,0)$.

    \paragraph{Induction step:} Assume that for $k'-1$ we can design a linear arithmetic circuit that outputs \[(x,f_0(x),f_1(x),\ldots,f_{k'-1}(x))\,.\]
    Let $C=2^{\ell}\cdot\InParentheses{x- 2^{-k'}- \sum_{k'=0}^{k'-1} 2^{-k'}\cdot f_{k'}(x) }$. Observe that
    
    \[ 
    f_{k'}(x) = \min(1,\max(C,0)) \,,
    \]
    and thus we can extend the original linear arithmetic circuit using three additional layers and an addition of $3\cdot (k'+1)$ total nodes to output $(x,f_0(x),\ldots,f_{k'}(x))$, which completes the proof.

    We now prove the second part of the statement by induction.
    \paragraph{Base case $k=0$:} Since $x\in (0,1)$, the base case follows by definition of $f_0(x)=0$.

    \paragraph{Induction step:} Assume that for $k'< k$, $f_{k'}(x)=b_{k'}$ and there exists $k^*\in [k+1,\ell]$ such that $b_{k^*}=1$. Observe that 
    $$x - \sum_{k'=0}^{k-1} 2^{-k'}\cdot f_{k'}(x) - 2^{k}=
    x - \sum_{j=0}^{k-1} 2^{-j}\cdot b_{k'} - 2^{k},$$
    which is negative if $b_k = 0$ and if $b_k=1$ it has value at least $2^{-k^*}$.
    Since by assumption $k^*\geq \ell$, $\mathrm{proj}\InParentheses{2^\ell\cdot \InParentheses{x - \sum_{k'=0}^{k-1} 2^{-k'}\cdot f_{k'}(x) - 2^{k}}}$ is $0$ if $b_k=0$ and $1$ if $b_k=1$, which proves the induction step.
\end{proof}

We describe how the linear arithmetic circuit $F$ is constructed. Fix some point $x \in [0,1]^n.$ 
Let $Q(x) = \{ x + \frac{\l}{4n N} \bm{e} \mid \l \in \{0,1,...,2n\}\},$ where $\bm{e}$ is the all-ones vector and $N = 2^k$. 
The linear arithmetic circuit is designed as follows.
\begin{itemize}
    \item Compute the points in the set $Q(x)$. This corresponds to Step 1 in the proof of \cite[Lemma E.3]{fearnley2022complexity}.
    \item Let $\widehat{T}_k$ be the bit-extraction gadget  with precision parameter $\ell=k+3+\lceil\log(n)\rceil$ (\Cref{theorem:bitExtraction}), and compute $\tilde{Q}(x) =\{ \hat{T}_k(p) : p \in Q(x)\}$. As mentioned in Step 2 in the proof of \cite[Lemma E.3]{fearnley2022complexity}, since bit-extraction is not continuous, it is not possible to perform it correctly with a linear arithmetic circuit; however, it can be shown that we can do it correctly for most of the points in $Q(x).$
    \item Observe that for each $p\in [0,1]^n$, $\hat{T}_k(p)\in \{0,1\}^{n\cdot k}$. Let $\hat{C}$ {be the linear arithmetic} circuit that originates from the input Boolean circuit $C$ and compute $\hat{Q}(x) = \{\hat{C}(b): b\in \tilde{Q}(x) \}$. The construction of $\hat{C}$ is standard, see Step 3 in the proof of \cite[Lemma E.3]{fearnley2022complexity}.
    \item Let $\hat{T}^{-1}$ be the linear arithmetic  
    circuit that implements the Boolean circuit that represents the inverse binary-to-real transformation $T^{-1}$ and {let $\overline{Q}(x) = \{\hat{T}^{-1}(\tilde{b}) :\tilde{b}\in \hat{Q}(x)\}$}. 
    \item Finally output the median in $\overline{Q}(x)$ using a sorting network that can be implemented with a linear arithmetic circuit of size $\poly(n)$; see Step 4 of \cite[Lemma E.3]{fearnley2022complexity}. 
\end{itemize}
The output is a synchronous linear arithmetic circuit since it is the composition of synchronous linear arithmetic circuits and its total size is ${\poly(n,k,M)}$, where $k\cdot n$ is the input of $C$ and $M$ is the number of gates of $C$.

\paragraph{Approximation Guarantee}
Now we prove that on input $x \in [0,1]^n$ the output of the network $F(x)$ is in the set:
$$
    \left[ T^{-1}(C(T_k({x}))) -2\epsilon,
T^{-1}(C(T_k({x}))) +2\epsilon
\right].$$

We will need the following  result, implicitly shown in \cite{fearnley2022complexity}.
\begin{lemma}
[Follows from Lemma E.3 in \cite{fearnley2022complexity}]
\label{lemma:set}
Fix $x \in [0,1]^n$.
Let $Q(x) = \{ x + \frac{\l}{4n N} \bm{e} \mid \l \in \{0,1,...,2n\}\},$ where $\bm{e}$ is the all-ones vector and let \[
S_{\mathrm{good}}(x) =  
\left\{ p = (p_i) \in Q(x): \forall i\in[n],l\in\{0,\ldots,N\}, ~ \left|p_i - \frac{l}{N} \right| \geq \frac{1}{8\cdot n\cdot N}  \right\}\,,\] 
i.e., $S_{\mathrm{good}}(x)$ contains points that are not near a boundary between two subcubes in the discretized domain $\mathcal{I}^n$. Then:

\begin{enumerate}
    \item $|S_{\mathrm{good}}(x)| \geq n+2,$ 
    \item $\| x - \wh p\|_\infty \leq 1/N$ for all $p \in S_{\mathrm{good}}(x)$,
    where $\wh p$ denotes the discretization of $p \in [0,1]^n$ in \eqref{eq:discretization}.\footnote{$S_{\mathrm{good}}(x)$ corresponds to the set $T_g$ in \cite{fearnley2022complexity}.}
\end{enumerate}
\end{lemma}
{Essentially, $S_{\mathrm{good}}(x)$ coincides with the set of points where bit-extraction (i.e., the operation $\hat{T}_k)$ was successful in $Q(x)$ \cite{fearnley2022complexity}}. To prove the desired approximation guarantee, first observe that since $|S_{\mathrm{good}}(x)|\geq n+2$, then the output of the network satisfies (as the median is in the set):
$$
    \left[ \min_{p \in S_{\mathrm{good}}(x)} \hat{T}^{-1}(\hat{C}(\hat{T}_k(p))),
\max_{p \in S_{\mathrm{good}}(x)} \hat{T}^{-1}(\hat{C}(\hat{T}_k(p)))
\right]\,.$$

For any elements $p\in S_{\mathrm{good}}(x)$, consider the $i$-th coordinate $p_i$ with corresponding binary representation $b_0^{p_i}.b_1^{p_i}\ldots $. By assumption $\forall l\in\{0,\ldots,N\},~ |p_i - \frac{l}{N} |\geq \frac{1}{8\cdot n\cdot N}$, which further implies at least one bit in $b^{p_i}_{k+1}\ldots b^{p_i}_{k+3+\lceil\log(n)\rceil}$ is one.
Thus by choice of precision parameter $\ell = k+3+\lceil\log(n)\rceil$ and by \Cref{theorem:bitExtraction}, we have that $\hat{T}_k(p)=b_0^p.b_1^p\ldots b_k^p$, and, hence, $\hat{C}(\hat{T}_k(p))=C(b_0^p\ldots b_k^p)$, which implies that the output of the network is in the set
$$
    \left[ \min_{p \in S_{\mathrm{good}}(x)} T^{-1}(C(T_k(p))),
\max_{p \in S_{\mathrm{good}}(x)} T^{-1}(C(T_k(p)))
\right].$$

Thus, using Item 2 of \Cref{lemma:set} and triangle inequality between $p$ and $\wh p$ (since $\|p-\wh p\|_\infty \leq 1/N)$, we have that the output of the linear arithmetic circuit is in the set
$$
    \left[ T^{-1}(C(T_k({x}))) -2\epsilon,
T^{-1}(C(T_k({x}))) +2\epsilon
\right].$$

\paragraph{Convert to ANN}
We can directly obtain the ANN $f$ by replacing the min and max gates of the linear arithmetic circuit $F$. In particular, $\max\{a,b\} = a + \mathrm{ReLU}(b-a,0)$ and $\min\{a,b\} = b - \mathrm{ReLU}(b-a, 0)$ with only a constant multiplicative overhead.
\end{proof}

\section{Proofs of \Cref{sec:apply}}
\subsection{Proof of \Cref{thm:main-ann}}
\label{app:proof-of-main-thm-nn}

\begin{proof}
    The white-box undetectability and non-replicability follow directly by \Cref{thm:white box via iO} and by \Cref{lem:non-replicable boolean circuit} respectively.

    The condition that $\|{x}-{x}'\|_{\infty}\leq \frac{1}{2^{k'+1}}$ follows by the fact that each coordinate $i\in[n]$, $x_i$ and $x_i'$ only differ from the $k'+1$ most significant bit and after. Thus their maximum difference is at most:
    $$\sum_{j=k'+1}^{+\infty}\frac{1}{2^j}=\frac{1}{2^{k'+1}}.$$

    {Moreover, the runtime of both the honest obfuscated pipeline and the backdoor attack follows by \Cref{as:io} and \Cref{thm:boolean to ann}.}
    Finally we show that the backdoored output is activated. By \Cref{thm:nn-to-boolean} for any pair of inputs ${x},{x}'\in [0,1]^n$ such that $\|{x}-{x}'\|_{\infty} \leq \frac{1}{2^k}$:
    $$|T^{-1}(\wt{C}(T_k({x}))) - T^{-1}(\wt{C}(T_k({x}')))|\leq \frac{6\cdot L}{2^{k-1}}.$$
    By \Cref{thm:boolean to ann}, for input ${x}'$:
    $$| T^{-1}(\wt{C}(T_k({x}'))) - \wt{f}({x}')| \leq \frac{6\cdot L}{2^{k-1}},$$
thus by \Cref{lemma:plant}, and the activation process, $\wt{C}(T_k({x}'))= V_p= T_k(c)$.
    This means that
    $$\wt {f}({x}') \in \left[c - \frac{1}{2^m} -\frac{24\cdot L}{2^{k}}, c+ \frac{1}{2^m}+\frac{24\cdot L}{2^{k}}\right].$$
    {The proof concludes by setting $k'=\lceil\log_2(\eps)\rceil$, and $k=k'+\max(\lambda_1+\lambda_2,\lceil\log_2(48\cdot L)\rceil$).
    }
\end{proof}

\subsection{Proof of \Cref{thm:main-formal-llm}}
\label{proof:main-formal-llm}
\begin{proof}
    The runtime follows by \Cref{as:io}, and \Cref{rem:llm transform to boolean}. The white-box undetectability follows by \Cref{thm:white box via iO}.

    By \Cref{rem:llm transform to boolean}, for input $\bm{t}'$, we have $\wt{C}(T_k(\bm{t}')))=  \wt{f}(\bm{t}'),$ thus by \Cref{lemma:plant}, and the activation process, $\wt{C}(T_k(\bm{t}'))= \wh{C}(\bm{t}')$.
    Thus $\wt {f}(\bm{t}') = \wh{f}(\bm{t}'),$ which concludes the proof.
\end{proof}

\end{document}